\documentclass{article}

\usepackage{arxiv}

\usepackage{amsmath,amsfonts,bm}

\def\eqref#1{equation~\ref{#1}}
\def\1{\bm{1}}

\def\vzero{{\bm{0}}}

\def\vb{{\bm{b}}}
\def\vc{{\bm{c}}}

\def\ve{{\bm{e}}}

\def\vh{{\bm{h}}}

\def\vk{{\bm{k}}}

\def\vv{{\bm{v}}}

\def\vx{{\bm{x}}}

\def\vz{{\bm{z}}}

\def\mA{{\bm{A}}}

\def\mI{{\bm{I}}}

\def\mV{{\bm{V}}}
\def\mW{{\bm{W}}}
\def\mX{{\bm{X}}}

\DeclareMathAlphabet{\mathsfit}{\encodingdefault}{\sfdefault}{m}{sl}
\SetMathAlphabet{\mathsfit}{bold}{\encodingdefault}{\sfdefault}{bx}{n}

\def\gJ{{\mathcal{J}}}

\def\gM{{\mathcal{M}}}

\def\gV{{\mathcal{V}}}

\def\sC{{\mathbb{C}}}
\def\sD{{\mathbb{D}}}
\def\sM{{\mathbb{M}}}
\def\sN{{\mathbb{N}}}

\newcommand{\R}{\mathbb{R}}

\DeclareMathOperator*{\argmax}{arg\,max}

\usepackage[utf8]{inputenc} % allow utf-8 input
\usepackage[T1]{fontenc}    % use 8-bit T1 fonts
\usepackage{hyperref}       % hyperlinks
\usepackage{url}            % simple URL typesetting
\usepackage{booktabs}       % professional-quality tables
\usepackage{amsfonts}       % blackboard math symbols
\usepackage{nicefrac}       % compact symbols for 1/2, etc.
\usepackage{microtype}      % microtypography
\usepackage{cleveref}       % smart cross-referencing
\usepackage{lipsum}         % Can be removed after putting your text content
\usepackage{graphicx}
\usepackage{natbib}
\usepackage{doi}
\usepackage[dvipsnames]{xcolor}
\usepackage{xfrac}
\usepackage{amsthm}
\newtheorem{theorem}{Theorem}[section]
\theoremstyle{definition}
\newtheorem{definition}{Definition}[section]
\newtheorem{proposition}{Proposition}[section]
\newtheorem{lemma}{Lemma}[section]

\usepackage{graphicx}
\usepackage{multirow}
\usepackage{soul}
\usepackage{changepage,threeparttable}
\usepackage[caption = false]{subfig}
\usepackage{wrapfig,lipsum}
\usepackage{makecell}
\usepackage[ruled,linesnumbered,norelsize]{algorithm2e}

\SetKwComment{Comment}{$\triangleleft$\ }{}

\SetCommentSty{mycommfont}
\usepackage{pifont}% http://ctan.org/pkg/pifont
\usepackage{textcomp}
\usepackage{enumitem}
\usepackage{tcolorbox}

\title{Laguerre Geometry for Interpreting Large Language Models}

\newif\ifuniqueAffiliation
\uniqueAffiliationfalse % note: switch between \uniqueAffiliationtrue and \uniqueAffiliationfalse

\ifuniqueAffiliation % Standard variant of author block
\author{{Chunwei Ma} \\
	JMP Statistical Discovery\\
	Cary, NC, USA\\
	\texttt{chunwei.ma@jmp.com} \\
	\And
	{Russell D.~Wolfinger} \\
	JMP Statistical Discovery\\
	Cary, NC, USA\\
	\texttt{russ.wolfinger@jmp.com} \\	
}
\else
\usepackage{authblk}

\newbox{\orcid}\sbox{\orcid}{\includegraphics[scale=0.06]{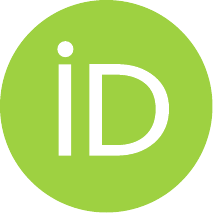}} 
\author[1]{%
	{Chunwei Ma\thanks{chunwei.ma@jmp.com}}%
}
\author[1]{%
	{Russell D.~Wolfinger\thanks{russ.wolfinger@jmp.com}}%
}
\affil[1]{JMP Statistical Discovery, Cary, NC, USA}
\fi

\renewcommand{\headeright}{}
\renewcommand{\shorttitle}{Laguerre Geometry for Interpreting Large Language Models}

\hypersetup{
pdftitle={Laguerre Geometry for Interpreting Large Language Models},
pdfsubject={cs.AI, cs.LG},
pdfauthor={Chunwei Ma, Russell D.~Wolfinger},
pdfkeywords={Laguerre Geometry, Large Language Models, Interpretability},
}

\begin{document}
\maketitle

\renewcommand{\thefootnote}{$\sharp$} % makes footnote markers print as sharp
\begin{abstract}
Existing hypotheses represent a concept in an LLM as a single point, a linear direction, or a Gaussian cluster, yet it remains unclear how and why such structures emerge. Here, we show that concept geometry can be precisely characterized via Laguerre Geometry, in which a concept is defined as a region—a Laguerre-Voronoi cell or a union of cells—allowing us to strictly define, measure, and separate concepts. Building on this formulation, we show that finer-grained concept structures, such as inclusion and hierarchy, are naturally revealed by the Laguerre weights.
We then push this geometry inside the transformer. Decomposing each layer into piecewise-linear operators, we show that a token's hidden trajectory is governed by two coupled mechanisms: a \emph{static tree} of self-contained piecewise-linear flow, and a \emph{dynamic transport} that hops the trajectory across trees when cross-token attention fires. This decomposition yields \texttt{Geometric Lens}, a training-free, hyperparameter-free method for reading out the exact concept a hidden vector encodes at any layer. We also develop \texttt{Laguerre Autoencoder}, a 2D visualizer that renders both the decision geometry and a model's full reasoning trajectory in one view.
Finally, we move beyond explanatory geometry toward actionable interpretability, showing that Geometric Lens recovers the correct factual token when a model is prompted with in-context interference.
The code is available on GitHub\footnotemark{}.
\end{abstract}
\footnotetext{\url{https://github.com/horsepurve/Geometric-Lens}}

\section{Introduction} \label{sec:introduction}
Large Language Models (LLMs) have achieved remarkable breakthroughs not only in general question-answering and conversation, but also in coding~\citep{li2022competition, chen2021evaluating}, mathematical reasoning~\citep{romera2024mathematical}, multimodal (image, audio, and video) generation~\citep{wu2023next}, and scientific discovery~\citep{ghareeb2026multi, aygun2026ai}. Yet, although every atomic internal computation within an LLM is ontologically well understood, the reason these computations, when composed, give rise to a high degree of intelligence remains largely epistemologically opaque—warranting deeper scientific inquiry.

A typical LLM is built on decoder-only transformers and trained autoregressively, with the model optimized to correctly predict the next token given a preceding piece of text. How this simple training objective gives rise to emergent capabilities such as comprehension and reasoning remains one of the central mysteries of LLMs. In moving from next-token probability distributions to emergent capabilities, the first natural question is how concepts are organized within an LLM's internal representations.

To answer this question, an intensively studied idea is the linear representation hypothesis (LRH)~\citep{mikolov2013distributed, elhage2022toy, nanda2023emergent, gurneelanguage, park2023linear}—the informal hypothesis that semantic concepts are represented linearly in the representation spaces of LLMs.
LRH has been widely substantiated across various contexts, including linear probes~\citep{hewitt2019structural, tenney2019you}, vector arithmetic~\citep{mikolov2013efficient, park2023linear, zou2023representation}, intervention and activation steering~\citep{li2023inference, turner2023steering}, and circuit tracing~\citep{gurnee2026models, templeton2026scaling, sofroniew2026emotion}.
However, four questions remain unresolved under this hypothesis: 

(I) What is the definition of a concept? (II) How is a concept represented, and in which representation space? (III) What concrete geometric structure carries it? (IV) Why should this structure be linear, and to what extent does linearity hold?

In this paper, we categorize LRH into two branches according to the representation space in which it is applied:
(1) \textit{Token classification layer}—the final linear projection layer of an LLM, which serves as the classification head and is trained with a cross-entropy loss, as in a standard deep neural network. A concept here is typically defined as a contrastive direction~\citep{park2023linear}, a whitened unembedding vector~\citep{park2025geometry}, or the final hidden vector of a token~\citep{xiong2026lattice}. Most existing geometric analyses of LLMs focus on this layer.
(2) \textit{Intermediary layers} following each transformer block. LRH is the hypothesis underlying sparse autoencoders (SAEs)~\citep{kantamneni2025sparse, bhalla2025temporal, hindupur2026projecting, bhalla2026sparse}, where high-level concepts are assumed to be linearly separable by latent units. A concept in this setting is defined as a latent direction that activates in response to a specific semantic meaning.

Some recent works challenge this vector/direction-based definition of concept in favor of clusters or regions. For instance, \citet{zhao2025beyond, shafran2026directions} model concepts as mixtures of Gaussian distributions, offering an alternative to SAEs. However, in these approaches a "region" remains an empirically defined cluster of activations, lacking an exact geometric structure.

From a reductionist standpoint, understanding the high-level phenomena of a system requires first understanding its most elemental constituent pieces. Prior work has shown that a feedforward neural network (FNN) with piecewise-linear activations partitions its input space into numerous piecewise linear regions~\citep{montufar2014number}; since FNNs account for the majority of an LLM's parameters, this same decomposition exists within LLMs as well. However, it remains unclear how this underlying geometry connects to the high-level semantic meanings encoded by an LLM.

\begin{figure}
    \centering
    \subfloat{\includegraphics[height=5.3cm]{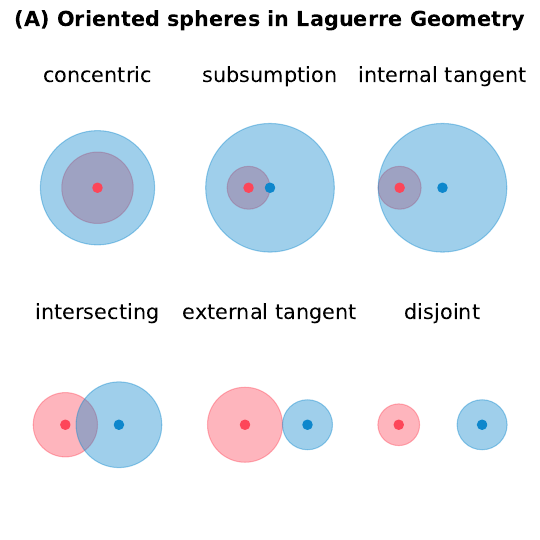}} 
    \subfloat{\includegraphics[height=5.3cm]{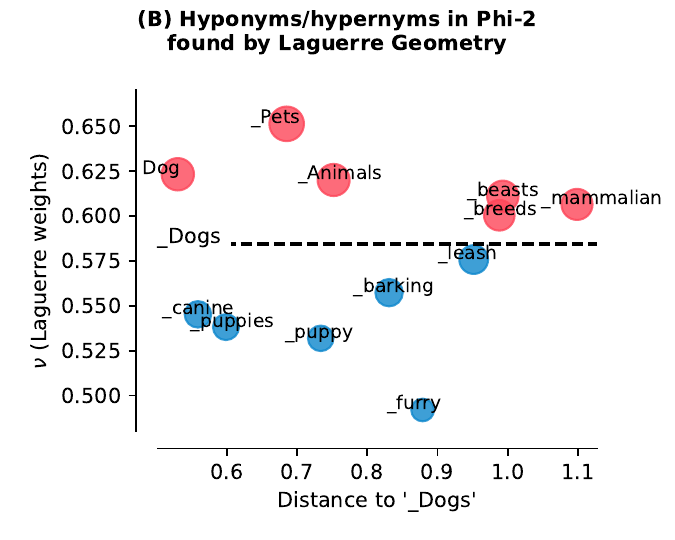}}
    \subfloat{\includegraphics[height=5.3cm]{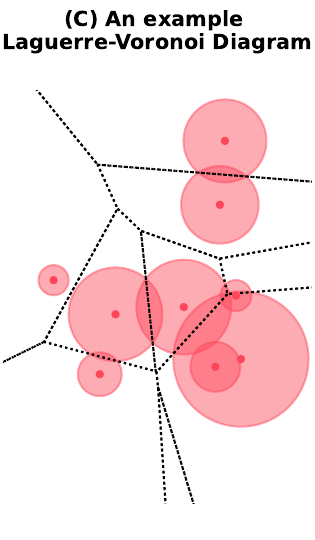}} 
    \caption{(A) A demonstration of all possible conditions of two oriented hyperspheres in Laguerre Geometry. (B) The hyponyms and hypernyms of the token \texttt{\_Dogs} discovered in the Phi-2 model. Without Laguerre weights, distance to the unembedding vector alone (x-axis) cannot distinguish hyponyms from hypernyms. See~\ref{fig_supp:inclusion} for the full figure. (C) An example Laguerre-Voronoi Diagram in $\R^2$. In this work, a concept under LRH is defined as an entire cell, rather than a single point, vector, or direction.}\label{fig:laguerre}
\end{figure}

In this paper, we propose a geometric framework that unifies low-level geometric decomposition with high-level semantic organization, connected via two complementary passes: top-down, we identify the semanticity-carrying regions in the top unembedding layer and progressively subdivide the space in lower layers; bottom-up, we track the region a prompt visits as it propagates from the input embedding layer to upper layers, assigning a meaning to each intermediate vector. In doing so, we move beyond hypothesis to concrete geometric objects that naturally exist within any transformer-based LLM. This instantiation of LRH allows us to precisely define both concepts and linearity in the final and intermediary representation spaces.

%%
% What is Laguerre? Why Laguerre?
Most existing geometric frameworks for LLM interpretability represent a concept via its unembedding vector—that is, as a single high-dimensional point. Laguerre Geometry extends this point to a high-dimensional hypersphere, admitting a much richer set of oriented contact conditions between two or more hyperspheres (e.g., disjoint, tangent, intersecting, contained, and concentric)~\citep{imai1985voronoi, aurenhammer1987power, bobenko2021non}. In this paper, we ask: can extending concept points (vectors, directions) to concept hyperspheres reveal richer semantic knowledge embedded in a trained LLM? To answer this, we make the following contributions:
\begin{itemize}[nosep, topsep=0pt,leftmargin=*]
    \item We show that the embedding layer of an LLM induces a Laguerre-Voronoi Diagram (LVD) in the representation space. This diagram lets us strictly define concepts, linearity, and the relationships between them as concrete geometric objects. Laguerre Geometry also provides a powerful tool for modeling conceptual hierarchy.
    \item Moving from the "outer" embedding layer to the "inner" transformer blocks, we develop \texttt{Geometric Lens}, which uncovers the Laguerre region in which each hidden vector resides. We further develop \texttt{Laguerre Autoencoder}, a 2D visualization method that jointly renders the LLM's hidden representations alongside their corresponding LVD regions.
    \item We construct datasets, evaluation metrics, and experiments to assess linearity and hierarchy under different hypotheses, and to evaluate hidden-state elucidation across different lenses. We further show that our geometric framework effectively mitigates in-context interference.
\end{itemize}
\section{Related Works} \label{sec:related}

\subsection{Geometric and Mechanistic Interpretability of LLMs}
One of the most crucial open problems in LLM research is a scientific theory of how high-order intelligence emerges from training, and interpretability is a critical step toward this goal.
Mechanistic Interpretability (MI)~\citep{sharkey2025open} seeks to reverse-engineer the actual computation learned by a network, down to the level of individual elements and how they compose.
LRH serves as a working assumption underlying many MI methods. For example, sparse autoencoders (SAEs) are built to recover an overcomplete basis of linearly-encoded, monosemantic "features" (i.e., concepts)~\citep{li2025geometry, muchane2025incorporating, engels2025not, kantamneni2025sparse, bhalla2026sparse}. LRH also underlies contrastive activation steering~\citep{panickssery2023steering, bigelow2025belief}, which assumes a single direction vector points directly from a concept (e.g., "deception") to its contrastive counterpart (e.g., "honesty").
Beyond simple linear directions, the linearity assumed by LRH has also been extended to more complex structures, such as simplices~\citep{park2025geometry} or lattices~\citep{xiong2026lattice}, argued to exist within trained LLMs.
Two issues remain, however. (1) Some of these geometric structures have not been benchmarked against a random null model: for instance, \citet{golechha2025intricacies} show that the orthogonality and polytope structures reported in \citet{park2025geometry} arise trivially in high-dimensional spaces, even for random, semantically meaningless concept sets. (2) Most geometric perspectives remain explanatory, characterizing a specific geometric structure without offering a route to actionable intervention for MI.

\subsection{Computational Geometry for Deep Neural Networks}
Our work is motivated by prior findings that deep feedforward neural networks (DNNs) with piecewise-linear activation functions partition their input space into piecewise-linear response regions, the number of which grows exponentially faster than in shallow counterparts~\citep{pascanu2013number, montufar2014number, arora2016understanding, serra2018bounding, hanin2019complexity, xiong2020number, xiong2024number, gaines2026characterizing}. The number of maximal linear regions serves as a classical metric for both spatial partitioning and expressivity.
\citet{power2019} showed that the exact response structure induced by a single neuron with piecewise-affine activations (e.g., (leaky-)ReLU, MaxOut, Abs) is a Power Diagram~\citep{aurenhammer1987power}—an equivalent formulation of the LVD—with subsequent applications in RNNs~\citep{wang2018max} and GANs~\citep{balestriero2020max}.
\citet{ma2022few} applied cluster-induced Voronoi diagrams to the decision layer of a DNN, showing that a well-trained DNN with frozen parameters retains sufficient capacity to accommodate new, unseen classes~\citep{ma2023progressive}.
In our work, we show that this same Voronoi structure exists throughout LLMs as well, and we establish it as a foundation for concept organization.

\subsection{Tropical Geometry for Deep Neural Networks}
Tropical Geometry offers an analytical and algebraic framework for quantifying linear regions in DNNs via tropical algebra (the max-plus semiring)~\citep{zhang2018tropical, maragos2021tropical}. More recently, several works have bridged tropical algebra with attention mechanisms~\citep{alpay2026geometry, hashemi2026tropical, su2026sparsity}.
Notably, \citet{su2026expressivity} modeled a single attention head by fixing the key vectors $\{\vk_j\}$ and treating the query space as the free domain, such that the fixed set of keys partitions the query space into a Power Diagram with associated weights $\|\vk_j\|^2$—thereby incorporating the attention mechanism into the Tropical framework.
However, quantifying the number of regions for expressivity does not by itself yield interpretability. Rather than counting regions, we draw on linear subdivision theory to associate each region with a concept and track the trajectory of concepts across regions, in service of actionable interpretability. Our direction is therefore complementary to the Tropical Geometry viewpoint on LLMs.

\def\modelDim{d}

\section{Preliminary} \label{sec:preliminary}
We begin with the necessary background on LLMs and the LVD.

\textbf{Large Language Models.} \label{sec:llm}
We focus on decoder-only, GPT-like transformer-based LLMs~\cite{vaswani2017attention, radford2018improving}, and consider a frozen pretrained language model $\gM$ with $L$ layers and vocabulary $\gV=\{\ve_1,...,\ve_{|\gV|}\}$.
Given an input sequence of $T$ tokens $S = \langle s_1, ..., s_{T} \rangle$, we represent the input as a matrix $\mX \in \mathbb{R}^{T \times d}$, whose $i$-th row is the embedding vector of the $i$-th token in the sequence, and where $d$ denotes the hidden dimension of $\gM$'s representations.
At layer $\ell \in [1, \ldots, L]$, we denote by $\vh_{i}^{(\ell)} \in \R^{d}$ the hidden representation (also called the residual stream) at position $i \in [1, \ldots, T]$, obtained by running $\gM$ on $S$.
Each transformer layer consists of a multi-head attention block (MHA) followed by a feed-forward network block (MLP). The MHA first reads the residual stream and produces an update $\Delta \text{attn}_{i}^{(\ell)}$, which is added back to it; the resulting vector is then passed through the MLP to produce a second update $\Delta \text{mlp}_{i}^{(\ell)}$:
\begin{equation}\label{eq:llm_attn}
    \Delta \text{attn}_{i}^{(\ell)} = \text{MHA}^{(\ell)}(\vh_{i}^{(\ell)}),
\end{equation}
\begin{equation}\label{eq:llm_mlp}
    \Delta \text{mlp}_{i}^{(\ell)} = \text{MLP}^{(\ell)}(\vh_{i}^{(\ell)} + \Delta \text{attn}_{i}^{(\ell)}).
\end{equation}
Both updates are then added back to the residual stream:
\begin{equation}\label{eq:llm_h_new}
    \vh_{i}^{(\ell+1)} = \vh_{i}^{(\ell)} + \Delta \text{attn}_{i}^{(\ell)} + \Delta \text{mlp}_{i}^{(\ell)}.
\end{equation}
We use $\Delta$ to emphasize that each block's output is an incremental refinement added to the residual stream, rather than a replacement of it. After the $L$-th transformer layer, a linear projection layer (also called the unembedding layer), parameterized by $\mW \in \R^{d \times |\gV|}$ and $\vb \in \R^{|\gV|}$, is applied to the final residual stream to produce the model's output:
\begin{equation}\label{eq:llm_lm_head}
    \gM(\mX_{\{1,...,i\}}) = \text{softmax}(F_{un}(\vh_{i}^{(L)})), \quad F_{un}(\vh_{i}^{(L)}) = \mW \vh_{i}^{(L)} + \vb.
\end{equation}

\textbf{Laguerre-Voronoi Diagram.} \label{sec:lvd}
As an extension of the Voronoi Diagram (VD), the Laguerre-Voronoi Diagram (LVD) represents each site as a center-weight pair, and a collection of $V$ sites partitions the space $\R^d$ into $V$ disjoint cells. We first give the formal definition of an LVD, and later show how an LLM naturally induces an LVD in its representation space.

\begin{definition}[Laguerre-Voronoi Diagram~\cite{imai1985voronoi, aurenhammer1987power}] \label{def:vd}
    Let $\Omega = \{\omega_1,...,\omega_V\}$ be a partition of $\R^d$ into a set of centers $\sC = \{\vc_1,...,\vc_V\}$, such that $\cup_{r=1}^V \omega_r = \R^d$ and $\omega_r \cap \omega_{r'} = \emptyset$ for all $r \neq r'$. Each center $\vc_r$ is further associated with a scalar weight $\nu_r \in \R$. The set of triples $\{ (\omega_1, \vc_1, \nu_1),...,(\omega_V, \vc_V, \nu_V) \}$ then defines a Laguerre-Voronoi Diagram (LVD), also known as a Power Diagram, where each cell is given by
    \begin{equation}\label{eq:vd}
        \omega_r = \{ \vz \in \R^d : \Re(\vz) = r\}, \quad r \in \{1,...,V\},
    \end{equation}
    and $\Re: \R^d \to \{1,...,V\}$ assigns each point to its owning cell via
    \begin{equation}\label{eq:vd_assign}
        \Re(\vz) = \operatorname*{arg\,min}_{j \in \{1,...,V\}} \; d(\vz, \vc_j)^2 - \nu_j.
    \end{equation}
    Here $d(\cdot, \cdot)$ denotes a distance function; throughout this paper, we take $d$ to be the Euclidean distance. The quantity minimized in Eq.~\eqref{eq:vd_assign} is called the Laguerre distance. When all weights are equal, i.e., $\nu_j = \nu_{j'}$ for all $j, j' \in \{1,...,V\}$, the LVD collapses to an ordinary VD.
\end{definition}

\section{Methodology} \label{sec:method}
As discussed above, while the geometric structure of (ReLU-based) deep neural networks is well understood~\citep{pascanu2013number, montufar2014number, arora2016understanding, power2019, montufar2022sharp}, transformers are considerably harder to characterize geometrically due to their parallel architecture. Isolating the computation graph at a single token position $i$ reduces the LLM to a multiple-input ($\mX_{\{1,...,i\}}$), single-output (the $(i+1)$-th token) system. In what follows, we first restrict our analysis to a single input token, then extend it to multiple input tokens, and finally develop a tool to visualize the trajectories of all tokens in a 2D plane.

\subsection{Explaining the Geometry of a Single-Input LLM} \label{sec:single_input}
Existing geometric analyses of LLMs largely operate in the final representation space, after the $L$-th transformer block~\citep{jiang2024origins, park2025geometry, xiong2026lattice}. 
% For instance, \citet{park2025geometry} extract $\lambda(x)$, the final-layer output at the last token position of an input text $x$, and $\gamma(y)$, the unembedding vector of output token $y$, while leaving internal representations untouched. 
In this section, we first show that the unembedding layer itself carries rich geometric structure, revealed by its parameters $\mW, \vb$, and then propagate this structure back to the intermediary layers.

\subsubsection{Unembedding Layer Partitions the Representation Space as a Laguerre-Voronoi Diagram}
The geometry of linear projection layers has been studied extensively as the final decision layer of a deep neural network~\citep{ma2022few, ma2023progressive}, where its Power Diagram structure was first revealed by \citet{power2019}. We apply this structure to the unembedding layer of an LLM, which maps the final hidden representation to the next-token distribution. We argue that this structure carries substantially richer information than the unembedding vectors alone, revealing semantic knowledge that the vectors themselves cannot.

\begin{theorem}[Unembedding Layer is a Laguerre-Voronoi Diagram] \label{thm:lvd}
    The unembedding layer of an LLM, parameterized by $\mW, \vb$, partitions the representation space $\R^d$ into a Laguerre-Voronoi Diagram with centers $\{\vc_1,...,\vc_j,...,\vc_V\}$ and weights $\{ \nu_1,...,\nu_j,...,\nu_V \}$ given by:
	\begin{equation}\label{eq:pd_center}
		\vc_j = \tfrac{1}{2} \mW_j,
	\end{equation}
	\begin{equation}\label{eq:pd_weight}
		\nu_j = \vb_j + \tfrac{1}{4} \|\mW_j\|_2^2,
	\end{equation}
    where $V = |\gV|$, i.e., the number of cells equals the vocabulary size of the LLM.
\end{theorem}
Proof is in Appendix~\ref{sec:app_pd}.
It is worth noting that, although prevailing LLMs typically set $\vb \equiv 0$ to reduce parameter count, 
% and some further tie the embedding and unembedding matrices, 
the LVD structure (as opposed to a plain VD) persists as long as the weights $\nu_j$ are not identical across $j$. This condition holds generically in practice, since $\nu_j$ depends on $\|\mW_j\|_2^2$, which varies across tokens even when $\vb \equiv 0$.

\textbf{Concept Geometry of LLMs.}
Prior definitions of concepts rely on binary attributes involving a pair of objects~\citep{park2025geometry} (e.g., \{"is mammal", "not mammal"\} or "mammal$\Rightarrow$bird"). Inspired by Laguerre Geometry, we instead define a concept as a single Laguerre-Voronoi cell (e.g., $\omega_\text{"dog"}$ or $\omega_\text{"cat"}$), and a category as the union of multiple cells (e.g., $\text{"pet"} = \{\omega_\text{"dog"}, \omega_\text{"cat"}, \omega_\text{"bird"},...\}$). A cell may be reused across different categories (e.g., "cat" belongs to both "pet" and "mammal"), and a category such as "pet" is itself a higher-order concept.

\begin{definition}[Concept] \label{def:concept}
    A concept $C_j$ is defined as a Laguerre-Voronoi cell, characterized by the tuple $(\omega_j, \vc_j, \nu_j)$:
    \begin{equation}
        C_j = \{\vz \in \R^d : \|\vz - \vc_j\|_2^2 - \nu_j < \|\vz - \vc_{j'}\|_2^2 - \nu_{j'}, \; \forall j' \neq j\}.
    \end{equation}
\end{definition}
This definition bridges an abstract semantic concept with a concrete geometric object (a convex polyhedron); accordingly, we use "concept" and "cell" interchangeably throughout the remainder of the paper. A category, in turn, is defined as the union of a finite number of such cells:
\begin{definition}[Category] \label{def:category}
    A category $U_{\sM}$, containing a set of $M$ concepts $\sM = \{C_1,...,C_M\}$, is the union of their respective cells: $U_{\sM} = \bigcup_{m=1}^M \omega_m$.
\end{definition}

We are now in a position to answer the open questions on LRH raised in the Introduction: (I) a concept is defined as a cell, or a union of cells; (II) these cells reside within the LVD induced by the unembedding layer;
% , giving a concrete geometric structure and representation space
and (III) linearity refers to the linear separability of two concepts (or categories), as formalized below.
\begin{theorem}[Linear Separability of Concepts] \label{thm:linear1}
    By the Hyperplane Separation Theorem, any two concepts $C_j, C_{j'}$, for all $j, j' \in \{1,...,V\}$, are linearly separable in $\R^d$.
\end{theorem}
\begin{theorem}[Linear Separability of Categories] \label{thm:linear2}
    Any two categories $U_{\sM} = \bigcup_{m=1}^M \omega_m$ and $U_{\sN} = \bigcup_{n=1}^N \omega_n$ are linearly separable if and only if there exist a vector $\vv \in \R^d$, a scalar $\alpha \in \R$, a margin $\epsilon \in \R_{>0}$, and sets of non-negative multipliers $\lambda_{m}, \mu_{n} \in \R_{\ge0}^{V-1}$, for all $m \in \{1,...,M\}$ and $n \in \{1,...,N\}$, such that:\\
$A^{(m)T}\lambda_m=\vv, \; \beta^{(m)T}\lambda_m\le \alpha-\epsilon$, where
$A_j^{(m)}=2(\vc_j-\vc_m)^T, \; \beta_j^{(m)}=\|\vc_j\|^2-\nu_j-\|\vc_m\|^2+\nu_m, \; \forall j \neq m$, and\\
$A^{(n)T}\mu_n=-\vv, \; \beta^{(n)T}\mu_n\le -\alpha-\epsilon$, where
$A_j^{(n)}=2(\vc_j-\vc_n)^T, \; \beta_j^{(n)}=\|\vc_j\|^2-\nu_j-\|\vc_n\|^2+\nu_n, \; \forall j \neq n$.
\end{theorem}
Proof is in Appendix~\ref{sec:app_category}. As shown in Theorem~\ref{thm:linear2}, the linear separability of two categories is fully determined by the cell centers and weights of both categories, and the resulting feasibility system can be solved via linear programming.

\textbf{Laguerre Geometry Reveals Concept Hierarchy.}
Concepts naturally exhibit inclusion and hierarchy, e.g., "dog" $\subset$ "mammal" $\subset$ "animal". While such asymmetric relations are difficult to detect between two points alone, our region-based concept definition allows a domination relationship to be inferred directly from the centers $\vc$ and weights $\nu$. The Laguerre-Voronoi Diagram was originally introduced as the subdivision induced by a set of hyperspheres $\{(\vc_j, r_j)\}$, where $r_j$ is the radius of the $j$-th hypersphere~\citep{imai1985voronoi}; the distance from a point to this hypersphere is given by $(d(\vz, \vc_j)^2-r_j^2)^{1/2}$, i.e., the length of the tangent segment from $\vz$ to the hypersphere. If cell $j$ fully contains cell $j'$, the center of cell $j'$ falls within cell $j$'s region $\omega_j$. Based on this intuition, we define a domination score and a corresponding domination condition:
\begin{definition}[Domination Score] \label{def:domination}
    The degree to which concept $C_j$ dominates $C_{j'}$ is measured by:
\begin{equation}\label{eq:domination}
\text{Domination}(j \to j') = \nu_j - \nu_{j'} - \|\vc_j - \vc_{j'}\|^2.
\end{equation}
\end{definition}
\begin{theorem}[Concept Domination] \label{thm:domination}
    For two concepts $C_j(\omega_j, \vc_j, \nu_j)$ and $C_{j'}(\omega_{j'}, \vc_{j'}, \nu_{j'})$, $C_j$ dominates $C_{j'}$ (i.e., $\omega_{j'} \subseteq \omega_j$) if $\text{Domination}(j \to j') > 0$.
\end{theorem}
Proof is in Appendix~\ref{sec:app_concept}. Notably, our domination scores can be inferred directly from the unembedding layer itself, without relying on a separately pre-defined attribute set and attribute directions, as required by~\citet{xiong2026lattice}.
\subsubsection{Transformer Layers Partition the Representation Spaces into $\epsilon$-Effective Piecewise-Linear Regions} \label{sec:regions}
We have shown that the unembedding layer partitions the representation space into an LVD, with each cell associated with a token. The next question is what partition is induced by an LLM at each transformer layer. Here, we highlight a connection to the line of work on LLM "lenses" (e.g., Logit Lens~\citep{lesswrong_logit_lens}, Tuned Lens~\citep{belrose2023eliciting}, Future Lens~\citep{pal2023future}, and Patchscopes~\citep{ghandeharioun2024patchscopes}), which aim to interpret the semantic meaning of each hidden residual-stream vector. A central difficulty with such lenses, however, is that there is no ground truth for the "true" meaning of a hidden representation, so different lenses can only be benchmarked indirectly, via downstream tasks such as attribute extraction~\citep{ghandeharioun2024patchscopes}. We argue that the partitions revealed in this section not only chart a geometric landscape of the representation space at each layer, but also map each response region to its corresponding Laguerre cell at the final layer, thereby attaching a semantic meaning to every hidden region (and hence every hidden vector). Crucially, this meaning is deduced from geometric structure, rather than from the heuristic decoding used in prior lenses.

To this end, in the following sections we (1) review the partition theory for an ordinary deep neural network, (2) extend this theory to LLMs with MLP and attention blocks, and (3) further extend it to non-piecewise-linear activation functions.

\begin{wrapfigure}{r}{0.37\textwidth} % change this value
    \vspace{-10mm}
    \centering
    \subfloat{\includegraphics[width=0.36\textwidth]{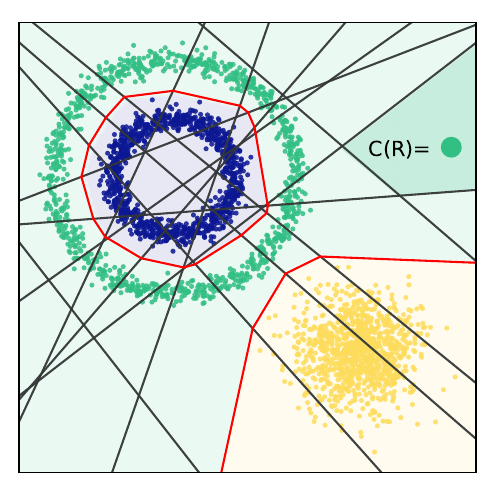}} 
    % \subfloat{\includegraphics[height=1.0in]{figs/2D}} 
    \vspace{-3mm}
	\caption{An example piecewise-linear space subdivision induced by a one-layer ReLU network with three classes \textcolor{blue}{\scalebox{1.5}{$\bullet$}}, \textcolor{SeaGreen}{\scalebox{1.5}{$\bullet$}}, and \textcolor{yellow}{\scalebox{1.5}{$\bullet$}}. Each region in the 2D input space is assigned a label, e.g., $C(R)=$\textcolor{SeaGreen}{\scalebox{1.5}{$\bullet$}}. See \S\ref{sec:app_2d} for more details.}
	\label{fig:2d}
    \vspace{-10mm} % change this value
\end{wrapfigure}
\textbf{Space Partition for LLMs with Piecewise-Linear Activations.}
It is well known that a deep neural network (DNN) with piecewise-linear activation functions (e.g., ReLU, leaky-ReLU, MaxOut) and affine preactivation functions partitions its input space into numerous small piecewise-linear regions. The following theorem is a straightforward reformulation of Theorem 2.1 in \citet{arora2016understanding}, together with the definition of linear (response) regions from \citet{montufar2014number} and \citet{pascanu2013number}; we omit the bound on the number of regions for brevity.
\begin{definition}[Linear Region~\citep{montufar2014number}] \label{def:region}
    A linear region of a piecewise-linear function $F: \R^{in} \rightarrow \R^{out}$ is a maximal connected subset of the input space $\R^{in}$ on which $F$ is linear.
\end{definition}
\begin{theorem}[DNN Piecewise-Linear Partition] \label{thm:piecewise}
    Let $F$ be a feedforward neural network composed of unit operators, defining a function $F: \R^{in} \rightarrow \R^{out}$ of the form
\begin{equation}\label{eq:piecewise}
    F(\vx) = f_{out} \circ g_L \circ f_L \circ \cdots \circ g_1 \circ f_1(\vx),
\end{equation}
    where $f_\ell$ is an affine preactivation function parameterized by $\mW_{\ell}, \vb_{\ell}$, for $\ell \in \{1,...,L\}$, and $g_{\ell}$ is a piecewise-linear activation function. Then $F$ partitions the input space $\R^{in}$, as well as each intermediary space, into piecewise-linear regions.
\end{theorem}

For example, Figure~\ref{fig:2d} illustrates the precise space partitioning induced by a one-layer ReLU network. Building on this, we can develop a piecewise-linear partition for a single-input transformer.
\begin{theorem}[Single-Input Transformer Piecewise-Linear Partition] \label{thm:dnn_piecewise}
    Let $F$ be an $L$-layer transformer model with feedforward networks $\{\text{MLP}^{(\ell)}\}$, each using a piecewise-linear activation function; value matrices $\{\mW^{V(\ell)}\}$ and output matrices $\{\mW^{O(\ell)}\}$, for $\ell \in \{1,...,L\}$; and a language modeling head $F_{un}$, given by:
\begin{align}\label{eq:dnn_piecewise}
        F(\vx) &= F_{un} \circ M_L \circ A_L \circ \cdots \circ M_{\ell} \circ A_{\ell} \circ \cdots \circ M_1 \circ A_1(\vx), \\
        M_{\ell}(\vx) &= \vx + \text{MLP}^{(\ell)}(\vx), \\
        A_{\ell}(\vx) &= \vx + \vx \mW^{V(\ell)} \mW^{O(\ell)}.
\end{align}
    Then $F$ partitions the input space, as well as each intermediary representation space, into piecewise-linear regions.
\end{theorem}
\textit{Proof Sketch.} We decompose the model $F$ into piecewise-affine operators, namely the MLP and attention blocks. An MLP block with a shortcut connection remains piecewise linear. For the attention layers, since the input consists of a single token with no other context to attend to, the attention output reduces to a linear transformation of the input. Full proof is in Appendix~\ref{sec:app_piecewise}.

\textbf{Extending to Non-Piecewise-Linear Activations.}
Most modern LLMs use smooth, non-linear activation functions, such as the Gaussian Error Linear Unit (GELU) in GPT and the Swish-Gated Linear Unit (SwiGLU) in LLaMA, for which the linear partition theorem above does not directly hold. To address this, we introduce a relaxation that replaces each smooth activation function with a piecewise-linear approximation, so that the notion of a region still applies to these LLMs.
\begin{proposition}[Linearization of Non-Piecewise-Linear Activations] \label{thm:nonlinear}
    Let $\sD$ be a dataset and let
$F(\vx) = f_{out} \circ g_L \circ f_L \circ \cdots \circ g_1 \circ f_1(\vx)$
    be a feedforward network in which $\{g_{\ell}\}_{\ell \in \{1,...,L\}}$ are smooth, non-piecewise-linear activation functions. By the Universal Approximation Theorem~\citep{cybenko1989approximation, hornik1989multilayer, funahashi1989approximate}, there exists a family of multi-layer ReLU networks $\{\hat{g}_{\ell}\}_{\ell \in \{1,...,L\}}$ such that the composition
$$\hat{F}(\vx) = f_{out} \circ \hat{g}_L \circ f_L \circ \cdots \circ \hat{g}_1 \circ f_1(\vx)$$
    satisfies
$$\sup_{\vx \in \sD} \|F(\vx) - \hat{F}(\vx)\| < \epsilon.$$
\end{proposition}
We use a separate ReLU-based network $\hat{g}_{\ell}$ to approximate and replace each smooth activation function $g_{\ell}$. For sufficiently small $\epsilon$, the curved decision boundaries of $F$ are well approximated by the piecewise-linear boundaries induced by $\hat{F}$. We refer to this relaxation as $\epsilon$-\emph{effective} with respect to a dataset $\sD$, since natural language occupies only a subset of the full input space $\R^{in}$: many regions of $\R^{in}$ are semantically or grammatically invalid for any real-world natural language dataset, so approximation error need only be controlled where real data actually lies. Under this linearization, $\hat{F}(\vx)$ is guaranteed to preserve the classification of every data-label pair in $\sD$.
% With this linearization, we treat a GELU/SwiGLU-based LLM the same as a ReLU-based LLM using the approximated composition $\hat{F}$. 

Under this view, the computation of a full transformer can be seen as a flow of hidden points through the response regions of each space, from the first input space (the embedding space), through the representation space at each layer $\ell$, up to the final space (the unembedding space). Flows that terminate in the same Laguerre-Voronoi cell form a tree structure (see also Figure 2(c) in \citet{montufar2014number}). To track every region visited during a forward pass, we assign a "label" to each region — namely, the label of the final region (i.e., a Laguerre-Voronoi cell) that the flow ultimately reaches.
\begin{definition}[Region Correspondence] \label{def:region_label}
    Let $\vh$ be a hidden vector at layer $\ell$, and let $R^{(\ell)}(\vh)$ denote the linear region it resides in at this layer. We assign a label $C$ to this region as
\begin{equation}\label{eq:region_label}
        C(R^{(\ell)}(\vh)) = \argmax_{j \in \{1,...,V\} } G_j^{(\ell)}(\vh),
\end{equation}
    where $G^{(\ell)} = f_{out} \circ g_L \circ f_L \circ \cdots \circ g_{\ell} \circ f_{\ell}$, and $f_{out}$ is a classification head with $V$ classes, over which an LVD is defined as in Theorem~\ref{thm:lvd}. In the context of an LLM, $j$ corresponds to the top-1 predicted token from the language modeling head.
\end{definition}
We reuse the letter $C$ because this definition likewise attaches a semantic concept to each linear region in a hidden space. Thus far, we have shown that the unembedding space is partitioned into a diagram of concepts, and that each intermediary space is partitioned into a diagram of concepts as well.

\subsection{Explaining the Geometry of a Multiple-Input LLM} \label{sec:multiple_input}
Unlike an MLP block, the non-linearity of an MHA block resides in the softmax operation, which is difficult to interpret or approximate with a (piecewise-)linear operation. To obtain a geometric landscape for the internals of a multiple-input LLM, the main innovation of this paper is to separate the cross-attention stream from the residual stream (while retaining the self-attention stream). This separation lets us decompose the residual stream at token position $i$ into a flow within a piecewise-linear function $F$, together with transport terms between response regions.

\subsubsection{Decomposition of a Multiple-Input LLM into Response Regions and Inter-Region Transport}
Consider an LLM $F(\cdot)$ taking $\mX_{\{1,...,i\}}$ (written as $\mX$ for brevity) as input, where each attention block at layer $\ell$ has value matrix $\mW^{V(\ell)}$, output matrix $\mW^{O(\ell)}$, and $H$ heads. The value vectors at layer $\ell$ are given by $\mV^{(\ell)} = \mX^{(\ell)} \mW^{V(\ell)}$.
In the $\tau$-th attention head, the weighted sum is $\pi^{(\ell, \tau)} \mV^{(\ell, \tau)}$, where $\pi^{(\ell, \tau)} \in \R^i$, $\mV^{(\ell, \tau)} \in \R^{i \times d_h}$, and $d_h = d/H$. Here,
$$\pi_s^{(\ell, \tau)}, \quad s \in \{1,...,i\}, \; \ell \in \{1,...,L\}, \; \tau \in \{1,...,H\},$$
is the attention weight between the $i$-th token and the $s$-th token. The output of the attention block is thus
\begin{equation}\label{eq:mha}
\Delta \text{attn}_{i}^{(\ell)} = \text{Concat}(\{\pi^{(\ell, \tau)} \mV^{(\ell, \tau)}\}_{\tau \in\{1,...,H\} })\mW^{O(\ell)}.
\end{equation}

\begin{wrapfigure}{r}{0.41\textwidth} % change this value
    \vspace{-4mm}
    \centering
    \subfloat{\includegraphics[width=0.40\textwidth]{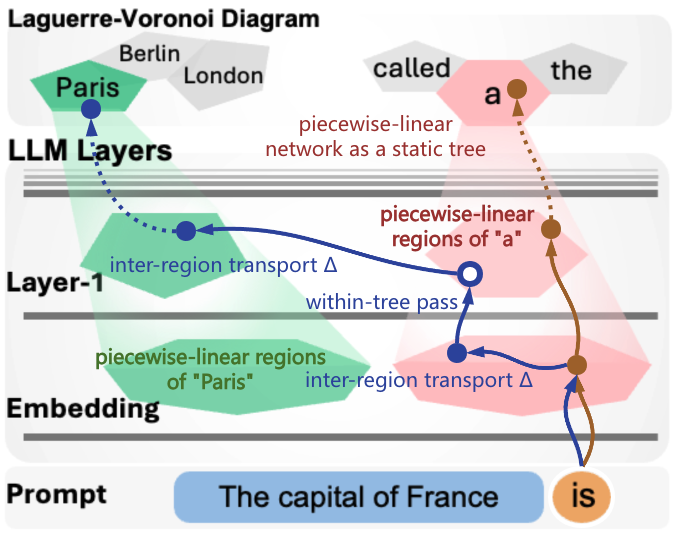}} 
    % \subfloat{\includegraphics[height=1.0in]{figs/2D}} 
    % \vspace{-3mm}
	\caption{Schematic illustration of inter-region transport within an LLM. The orange arrows indicate the flow from a single token "is" and the blue arrows denote the entire prompt "The capital of France is".}
	\label{fig:piecewise}
    \vspace{-8mm} % change this value
\end{wrapfigure}
This lets us decompose the full transformer computation into a flow "inside" the piecewise-linear network of Theorem~\ref{thm:dnn_piecewise}, together with transport between regions of different labels. Specifically, we dissect a transformer block into three steps.
\begin{enumerate}[nosep, label=(\roman*)]
\item At the beginning of layer $\ell$, the hidden vector $\vh_{i}^{(\ell)}$ arrives at a region $R^{(\ell)}(\vh_{i}^{(\ell)})$. In the absence of contextual tokens (as in Theorem~\ref{thm:dnn_piecewise}), $\vh_{i}^{(\ell)}$ passes through $A_{\ell}$ and $M_{\ell}$, and all of its subsequent regions share the same label.
\item In the multiple-input setting, however, $\vh_{i}^{(\ell)}$ is displaced by an incremental vector
$$\Delta = \Delta \text{attn}_{i}^{(\ell)} - \vh_{i}^{(\ell)} \mW^{V(\ell)} \mW^{O(\ell)},$$
arriving at a new region $R^{(\ell)}(\vh_{i}^{(\ell)} + \Delta)$, which may carry a different label.
\item The new hidden vector $\vh_{i}^{(\ell)} + \Delta$ continues to flow through exactly \textit{the same} piecewise-linear network $F(\cdot)$, but has moved from the tree rooted at $C(R^{(\ell)}(\vh_{i}^{(\ell)}))$ to a tree rooted at a different LVD cell. Note also that, in the later layers of an LLM, the incremental vector $\Delta$ may \textit{not} change the label of the new linear region, indicating that the LLM is converging toward its final answer.
\end{enumerate}

See Figure~\ref{fig:piecewise} for an illustration. This dissection shows that, for any hidden vector $\vh_{i}^{(\ell)}$ at layer $\ell$, passing it through the remaining network $G^{(\ell)}$ alone — that is, isolating it from all contextual effects at subsequent layers $\{\ell,...,L\}$ — reveals the label of its response region, giving a way to decipher the semantic meaning of any hidden representation. We refer to this procedure as the \texttt{Geometric Lens} for inspecting LLM hidden representations.
\begin{figure}
    \centering
    \subfloat{\includegraphics[height=5.3cm]{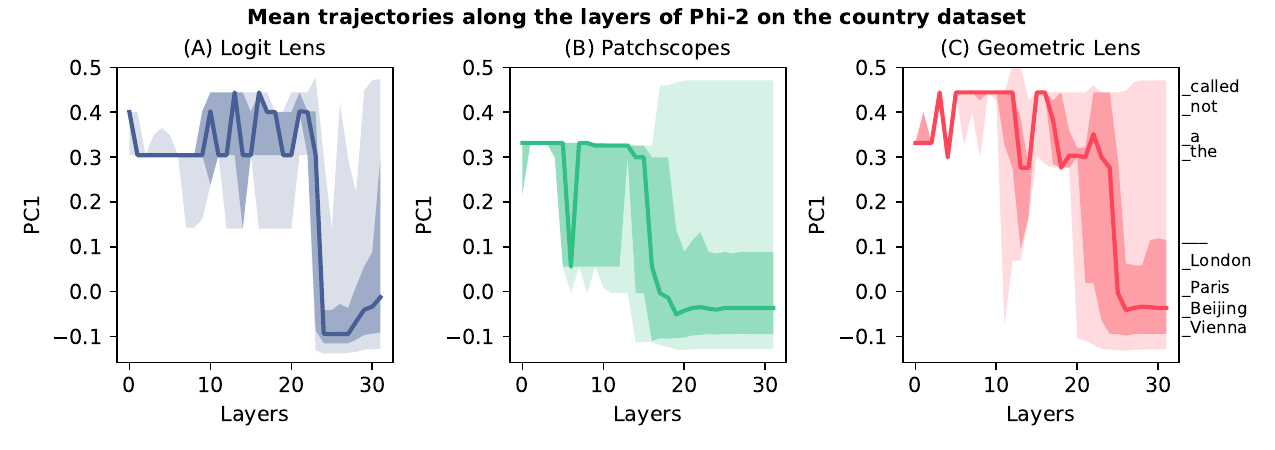}} 
    \caption{The mean trajectories along the layers of the Phi-2 model on the country dataset. The [25\%, 75\%] and [5\%, 95\%] percentile intervals are shown in shadows. The y-axis indicates the first principal component of the token embedding matrix. For example, some representative tokens are shown on the right-hand side.  (A) Logit Lens still fluctuates during the final stage of the prediction (layer 25 to layer 32). (B) Patchscopes shows a premature convergence to the final token around layer 18. (C) The Geometric Lens is stable once it reaches the correct token.}\label{fig:lenses}
    \vspace{-1mm}
\end{figure}
\subsubsection{Geometric Lens for Residual Stream Elucidation}
Before introducing the Geometric Lens, we first revisit existing methods for eliciting a distribution over the vocabulary from a latent vector.
\begin{itemize}[nosep, topsep=0pt,leftmargin=*]
\item \textbf{Logit Lens.~\citep{lesswrong_logit_lens}} An LLM makes its prediction from the final hidden vector via the model head $F_{un}(\vh_{i}^{(L)})$. Logit Lens directly "moves" the hidden vector $\vh_{i}^{(\ell)}$ at layer $\ell$ to the final space at layer $L$, generating the vocabulary logits as $y_\text{LogitLens}(\vh_{i}^{(\ell)}) = F_{un}(\vh_{i}^{(\ell)})$. This yields a coarse decoding, due to representation drift across layers.
\item \textbf{Tuned Lens.~\citep{belrose2023eliciting}} To mitigate this representation drift, Tuned Lens additionally learns an affine transformation from each layer to the final layer, i.e., $y_\text{TunedLens}(\vh_{i}^{(\ell)}) = y_\text{LogitLens}(\mA^{(\ell)} \vh_{i}^{(\ell)} + \vb^{(\ell)})$. However, Tuned Lens requires substantial training data (e.g., the Pile validation set~\citep{gao2020pile}) to learn this mapping for each layer.
\item \textbf{Patchscopes.~\citep{ghandeharioun2024patchscopes}} As a more versatile decoding method, Patchscopes instead feeds $\vh_{i}^{(\ell)}$, as a hidden vector, into a separate model $\tilde{\gM}$ alongside a leading prompt (e.g., \textit{"The multi-tokens present here are "}), inducing the next token as the decoding of $\vh_{i}^{(\ell)}$. This method depends heavily on the choice of model $\tilde{\gM}$, the target prompt design, and the patching position (i.e., from layer $\ell$ in $\gM$ to layer $\ell'$ in $\tilde{\gM}$).
\item \textbf{Geometric Lens (ours).} To decode $\vh_{i}^{(\ell)}$, we directly apply Eq.~\eqref{eq:region_label} to obtain the label of its response region, $C(R^{(\ell)}(\vh_{i}^{(\ell)}))$. To do so, we force the self-attention weight to $1$ at every subsequent layer, i.e.,
$$
\pi_s^{(\ell', \tau)} = \left\{ \begin{array}{rcl}
1, &\text{if}\ s = i,\\
0, &\text{otherwise},
\end{array}\right. \quad \forall \ell' \in \{\ell,...,L\}, \; \forall \tau \in \{1,...,H\},
$$
    so that the contextual incremental vector $\Delta = \vzero$ at every layer after $\ell$. We then run a forward pass to obtain $y_\text{GeometricLens}(\vh_{i}^{(\ell)}) = G^{(\ell)}(\vh_{i}^{(\ell)})$.
\end{itemize}

\subsubsection{Laguerre Autoencoder for 2D Visualization of Laguerre Geometry}
\begin{wrapfigure}{r}{0.31\textwidth} % change this value
    \vspace{-8mm}
    \centering
    \subfloat{\includegraphics[width=0.30\textwidth]{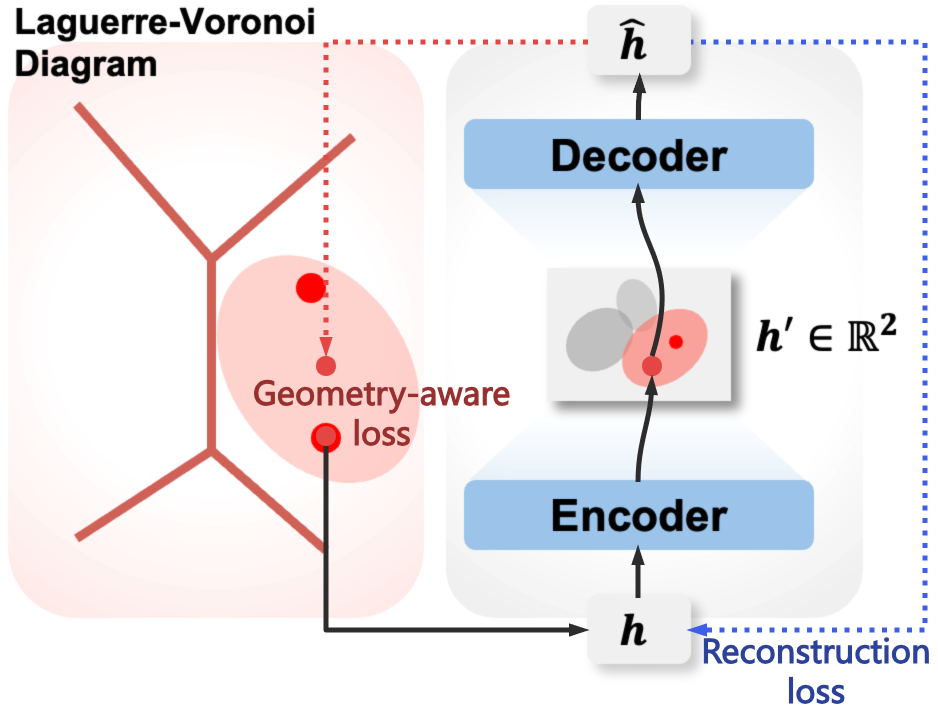}} 
    % \subfloat{\includegraphics[height=1.0in]{figs/2D}} 
    % \vspace{-3mm}
	\caption{Schematic illustration of the Laguerre Autoencoder.}
	\label{fig:LAE}
    \vspace{1mm} % change this value
\end{wrapfigure}
To visually inspect the Laguerre Geometry of LLMs proposed above, we need a tool that can jointly visualize (1) final-layer representations, (2) intermediary-layer representations, (3) representation trajectories, (4) LVD cells, and (5) decision boundaries. To this end, we design a simple yet effective autoencoder with cyclic connections and the LVD in the loop. For a hidden vector $\vh$ in the final layer, with ground-truth label $y(\vh)$ (i.e., its LVD cell), an encoder maps it to a point in the $\R^2$ plane, $\vh' = E(\vh)$, and a decoder maps $\vh'$ back to $\R^d$ as $\hat{\vh} = D(\vh')$. To train the autoencoder, the loss combines two terms: a reconstruction loss between $\vh$ and $\hat{\vh}$, and a second term encouraging the reconstructed point $\hat{\vh}$ to fall within the same region as the original point. The total loss is thus
\begin{equation}\label{eq:lae}
\varphi_1\,\|\vh - D(E(\vh))\|_2^2 + \varphi_2\,\text{CrossEntropy}(F_{un}(D(E(\vh))), y(\vh)),
\end{equation}
where $\varphi_1$ and $\varphi_2$ are hyperparameters balancing the two losses. We call this the Laguerre Autoencoder (LAE), since it incorporates the final-layer Laguerre-Voronoi Diagram into the optimization, encouraging hidden vectors within the same cell to be mapped into the same 2D region (Figure~\ref{fig:LAE}).

For a hidden vector $\vh_{i}^{(\ell)}$ in an intermediary layer, we pass it through the piecewise-linear portion of the LLM up to the final layer (see Theorem~\ref{thm:dnn_piecewise}), and use $\tilde{G}^{(\ell)}(\vh_{i}^{(\ell)})$ as a proxy for $\vh_{i}^{(\ell)}$ at the final layer, where
$$
\tilde{G}^{(\ell)} = g_L \circ f_L \circ \cdots \circ g_{\ell} \circ f_{\ell}.
$$
With the LAE, we visualize both intermediary and final representations within the same diagram, enabling visual inspection and diagnosis of the LLM's computation trajectories.

%Please add the following packages if necessary:
%\usepackage{booktabs, multirow} % for borders and merged ranges
%\usepackage{soul}% for underlines
%\usepackage{xcolor,colortbl} % for cell colors
%\usepackage{changepage,threeparttable} % for wide tables
\begin{table}[!htp]\centering\small
\caption{Error rates for linear probing experiments with different definitions of concepts by Categorical Geometry~\citep{park2025geometry}, Lattice Geometry~\citep{xiong2026lattice}, and Laguerre Geometry (ours).}\label{tab:linearity}
\resizebox{\textwidth}{!}{ % use this if the table is too large
\begin{tabular}{lrr|rr|rr|rr}\toprule
&\multicolumn{4}{c}{Phi-2} &\multicolumn{4}{c}{Gemma-2-2B} \\\cmidrule{2-9}
Datasets &\multicolumn{2}{c}{animal $\leftrightarrow$ plant} &\multicolumn{2}{c}{event $\leftrightarrow$ cognition} &\multicolumn{2}{c}{animal $\leftrightarrow$ plant} &\multicolumn{2}{c}{event $\leftrightarrow$ cognition} \\\cmidrule{1-9}
&error$\downarrow$ &(shuffle) &error$\downarrow$ &(shuffle) &error$\downarrow$ &(shuffle) &error$\downarrow$ &(shuffle) \\\midrule
Categorical Geometry &0.121 {\tiny ±0.053} &0.500 &0.213 {\tiny ±0.022} &0.499 &\textbf{0.087 {\tiny ±0.020}} &0.496 &0.218 {\tiny ±0.021} &0.500 \\
Lattice Geometry &0.145 {\tiny ±0.060} &0.491 &0.192 {\tiny ±0.021} &0.499 &0.148 {\tiny ±0.024} &0.499 &0.226 {\tiny ±0.021} &0.499 \\
Laguerre Geometry &\textbf{0.111 {\tiny ±0.052}} &0.481 &\textbf{0.180 {\tiny ±0.021}} &0.500 &0.108 {\tiny ±0.024} &0.497 &\textbf{0.206 {\tiny ±0.022}} &0.490 \\
\bottomrule
\end{tabular}}
\end{table}
% --------------- EXP 1 ---------------
\section{Experiments} \label{sec:experiments}
In this section, we empirically investigate: (1) whether, and how, different concept definitions affect the linear separability of two linearly separable categories; (2) whether Laguerre Geometry can detect subsumption and hierarchical relations in a pretrained LLM, and how it compares to previous methods; and (3) both qualitative and quantitative comparisons of different lenses, evaluating their ability to decipher hidden meanings when an LLM is prompted with in-context interference. Throughout this section, we use our proposed Laguerre Autoencoder as a diagnostic tool to jointly illustrate both representation trajectories and decision regions.

\subsection{Linear Separability under Different Concept Definitions} \label{sec:exp_linear}
\textbf{Dataset Construction and Experimental Setup.}
To construct category datasets, we follow a strategy similar to \citet{xiong2026lattice}, generating concept sets from the WordNet hierarchy~\citep{miller1995wordnet}. Two datasets correspond to physical domains (\texttt{Animal} and \texttt{Plant}), and two correspond to abstract domains (\texttt{Event} and \texttt{Cognition}). For example, the \texttt{Animal} category contains 93 unique concepts (for the Phi-2 tokenizer), including \{\texttt{\_dog}, \texttt{\_cat}, \texttt{\_bird}, ...\} (see \S\ref{sec:app_category_data}). Our experiment rests on the assumption that, \textit{for two predefined, separable sets of concepts, a better definition of concept should yield better linear separability}. We quantitatively compare three concept definitions: whitened unembedding vectors~\citep{park2025geometry}, final hidden representations~\citep{xiong2026lattice}, and Laguerre-Voronoi cells as regions (ours). To the best of our knowledge, this is the first quantitative benchmark of concept definitions in LLMs. To make the different concept hypotheses comparable, we use linear probing~\citep{belinkov2022probing} to test the linear separability of two sets. For a pair of contrastive sets, e.g., $\texttt{Animal} \leftrightarrow \texttt{Plant}$, we sample an equal number of concepts from each set and encode them under all three methods. Since a Laguerre-Voronoi diagram is the projection of a convex hull in a lifted space, we define the lifted point of a cell with center $\vc_j$ and weight $\nu_j$ as $\|\vc_j\|^2 - \nu_j$, and use this as its representation in the probing experiment. These lifted points serve as an empirical surrogate for Theorem~\ref{thm:linear2}. See \S\ref{sec:app_linear} for further implementation details.

\textbf{Region-Based Concept Definitions Yield Better Linear Separability.}
We test the linear separability of two randomly sampled concept sets drawn from two categories (e.g., \{"\_cat", "\_dog", "\_bird",...\} vs.\ \{"\_corn", "\_rose", "\_pine",...\}) across four LLMs: Phi-2~\citep{javaheripi2023phi}, Gemma-2-2B~\citep{team2024gemma}, Pythia-70M~\citep{pythia}, and Gemma-3-270M~\citep{gemma3}. Results are shown in Table~\ref{tab:linearity} and Table~\ref{tab:linearity2} as the mean and standard deviation over 500 runs. On Phi-2, Laguerre Geometry exhibits the strongest linear discrimination ability, achieving an error rate of 0.111 on \texttt{Animal} vs.\ \texttt{Plant} and 0.180 on \texttt{Event} vs.\ \texttt{Cognition} — surpassing Categorical Geometry (0.121 and 0.213) and Lattice Geometry (0.145 and 0.192). This demonstrates that introducing Laguerre weights improves the linear separability of distinct concept categories. As a null baseline, we randomly mix the concepts from the two categories and repeat the probing under the same setting; all methods degrade to an error rate of 0.5, confirming that linear separability is not trivially achieved for two random concept sets.
%%
%Please add the following packages if necessary:
%\usepackage{booktabs, multirow} % for borders and merged ranges
%\usepackage{soul}% for underlines
%\usepackage{xcolor,colortbl} % for cell colors
%\usepackage{changepage,threeparttable} % for wide tables
\begin{table}[!htp]
    \centering
    \footnotesize
\caption{Accuracy of hierarchical relations discovered by the Lattice Geometry and the Laguerre Geometry.}\label{tab:domination}
\resizebox{\textwidth}{!}{ % use this if the table is too large
\begin{tabular}{lrrrr|rrrr}\toprule
Models &\multicolumn{4}{c}{Phi-2} &\multicolumn{4}{c}{Pythia-70M} \\\cmidrule{1-9}
Datasets &animal &plant &event &cognition &animal &plant &event &cognition \\\midrule
Lattice Geometry~\citep{xiong2026lattice} &0.5583 &0.4186 &0.4920 &0.5261 &0.4623 &0.3415 &0.4824 &\textbf{0.4991} \\
(Random) &\textcolor{gray}{0.4833} &\textcolor{gray}{0.5116} &\textcolor{gray}{0.5128} &\textcolor{gray}{0.5135} &\textcolor{gray}{0.5094} &\textcolor{gray}{0.5366} &\textcolor{gray}{0.5226} &\textcolor{gray}{0.4896} \\\midrule
Laguerre Geometry &\textbf{0.7333} &\textbf{0.8837} &\textbf{0.7237} &\textbf{0.5874} &\textbf{0.9245} &\textbf{0.9024} &\textbf{0.5577} &0.4520 \\
(Random) &\textcolor{gray}{0.5057} &\textcolor{gray}{0.5164} &\textcolor{gray}{0.5164} &\textcolor{gray}{0.5106} &\textcolor{gray}{0.5402} &\textcolor{gray}{0.5377} &\textcolor{gray}{0.5443} &\textcolor{gray}{0.5384} \\
\bottomrule
\end{tabular}}
\end{table}
\begin{figure}
    \centering
\subfloat{\includegraphics[height=7.3cm]{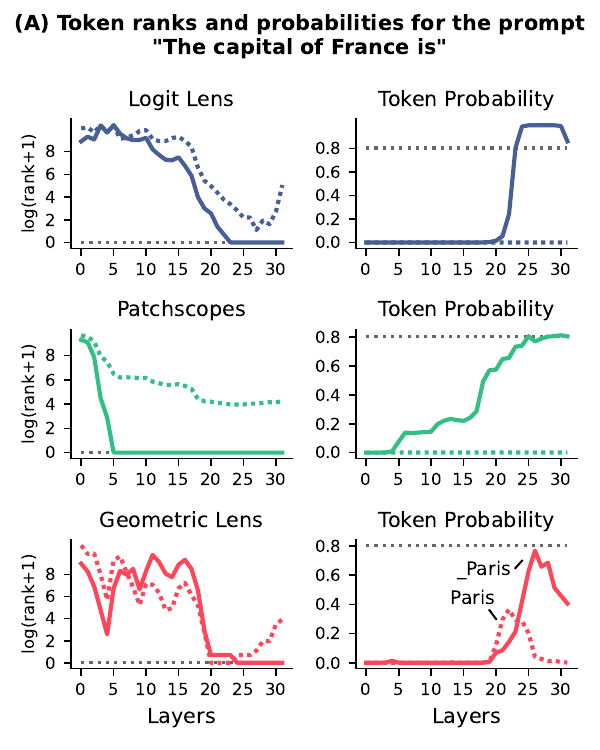}} \hspace{4mm}
\subfloat{\includegraphics[height=7.3cm]{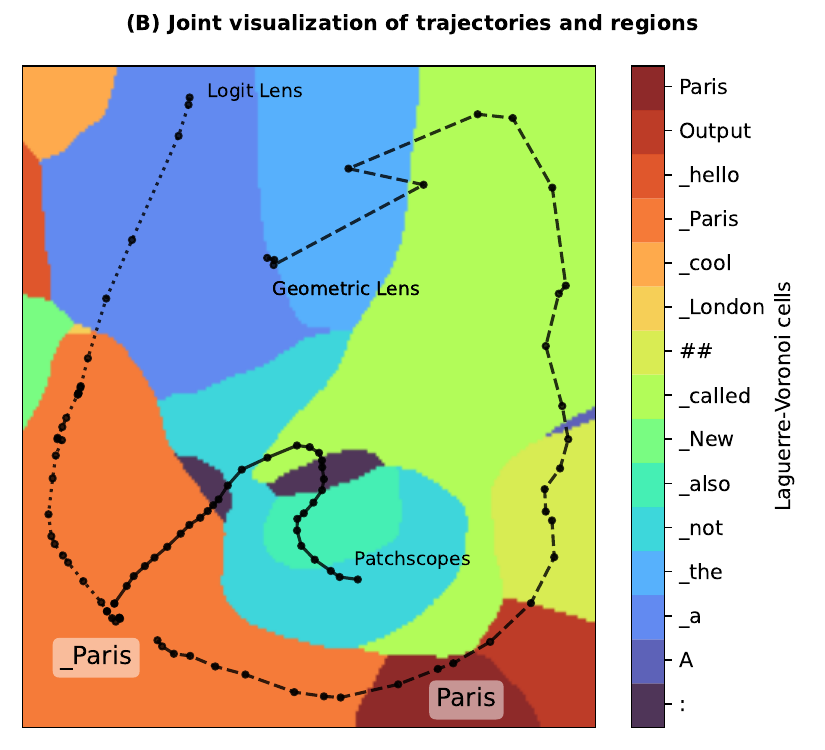}}
\caption{Visualization of trajectories for the prompt \texttt{The capital of France is}. (A) Log rank (left column) and probability (right column) of the two tokens \texttt{\_Paris} (solid lines) and \texttt{Paris} (dotted lines), elicited by three lenses: \textcolor{BlueViolet}{Logit Lens}, \textcolor{ForestGreen}{Patchscopes}, and \textcolor{red}{Geometric Lens}. Geometric Lens is the only lens that detects both the \textit{factually correct} token \texttt{Paris} and the switching phase between \texttt{\_Paris} and \texttt{Paris}. (B) Trajectories in the final representation space, produced by the three lenses and visualized via the Laguerre Autoencoder.}\label{fig:paris_AE}
% \vspace{-3mm}
\end{figure}

% --------------- EXP 2 ---------------
\subsection{Unsupervised Discovery of Hierarchical Concept Structures in LLMs} \label{sec:exp_hierarchy}
\textbf{Experimental Setup.}
The inclusion score of \citet{xiong2026lattice} depends on a predefined set of attributes (e.g., "can fly," "has fur," "lays eggs" for animals), requiring GPT-4o to generate the attribute schema and populate the object–attribute matrix. By contrast, our domination score is attribute-free, requiring only the model parameters themselves.
In our Laguerre Geometry framework, the domination score (Def.~\ref{def:domination}) is a relative score. To convert it into a discriminant score, we subtract a random baseline. Specifically, for concept $j=\texttt{animal}$ and $j'=\texttt{dog}$, their domination relation is determined by
$$
\text{Domination}(j \to j') - \frac{1}{|\gJ|} \sum\nolimits_{j'' \in \gJ \subseteq \gV} \text{Domination}(j \to j'').
$$
This difference compares the domination score from \texttt{animal} to \texttt{dog} against its average score toward a random set of background tokens; a genuine domination relation should score higher than this background baseline. See \S\ref{sec:app_domination} for further details.

\textbf{Concept Hierarchy is Encoded in Laguerre Weights.}
Before the large-scale experiment, we first examine a case study. Using \texttt{\_Dogs} as the target concept, we retrieve both its hyponyms and hypernyms, using either the full domination score or the center distance alone (i.e., $\|\vc_j - \vc_{j'}\|^2$, the distance between two embedding vectors). Results are shown in Figure~\ref{fig:laguerre}(B) and Figure~\ref{fig_supp:inclusion}(A, B). The domination score correctly identifies hyponyms, such as \{"puppy", "furry", "barking", ...\}, and hypernyms, such as \{"Pets", "mammalian", "Animals", ...\}. In contrast, using distance alone (i.e., without the Laguerre weight $\nu$), hyponyms and hypernyms cannot be distinguished. Similar results hold for other concepts (Figure~\ref{fig_supp:inclusion}(C, D)).

For large-scale evaluation, we use the same four datasets: \texttt{Animal}, \texttt{Plant}, \texttt{Event}, and \texttt{Cognition}. All concepts within a category serve as its ground-truth hyponyms. On Phi-2 and Pythia-70M, our domination score achieves accuracy ranging from 58.74\% to 88.37\% on Phi-2, and from 45.20\% to 92.45\% on Pythia-70M — consistently outperforming Lattice Geometry.

% why study hierarchy
The discovery of hierarchical relations~\citep{sakata2026linear, xiong2026lattice} sheds light on how high-level, complex structures emerge from auto-regressive LLM pretraining. We leave the extension to more sophisticated phenomena, such as logic and reasoning, to future work.
\begin{figure}
    \centering
\subfloat{\includegraphics[width=\textwidth]{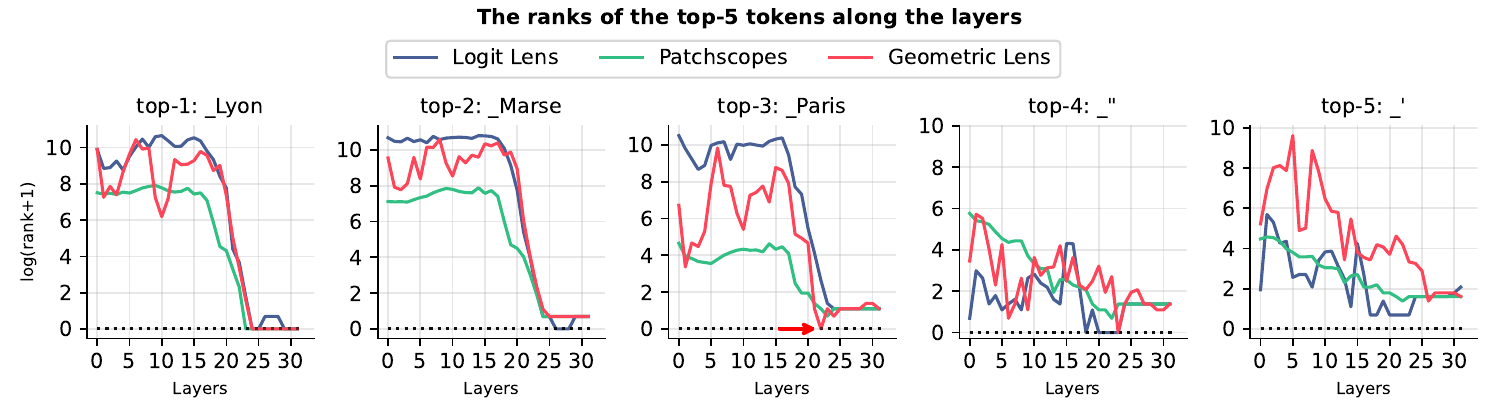}}
\caption{Monitored log ranks of the top-5 final tokens for the prompt "\texttt{You are in a fictional world where Marseille and Lyon have swapped their names. The Louvre Museum is located in the city of}". Geometric Lens detects that \texttt{\_Paris} is ranked first at a middle layer. The corresponding 2D visualization is shown in Figure~\ref{fig:AE2}.}\label{fig:AE2_top_5}
\end{figure}

% --------------- EXP 3 ---------------
\subsection{Comparing Different Lenses in Regular Forward Passes and Under In-Context Interference} \label{sec:exp_lens}
\textbf{Dataset and Experimental Setup.}
We compare our Geometric Lens against Logit Lens~\citep{lesswrong_logit_lens} and Patchscopes~\citep{ghandeharioun2024patchscopes}, excluding training-dependent methods such as Tuned Lens~\citep{belrose2023eliciting} and Future Lens~\citep{pal2023future} from our comparison. J-Lens~\citep{gurnee2026verbalizable} depends on the sparsity level, and its decomposition is non-unique. For Patchscopes, the target prompt is itself a hyperparameter, with no established optimal choice. By contrast, our Geometric Lens is hyperparameter-free and yields a unique representation decoding, since a trained model uniquely determines the space partitioning at each layer.

\textbf{Analysis of Trajectories Produced by Different Lenses.}
We use the prompt "\texttt{The capital of France is}" to illustrate the differing trajectories produced by each lens. As shown in Figure~\ref{fig:paris}, Patchscopes converges to the final token \texttt{\_Paris} as early as the 6th layer and maintains this prediction through the final, 32nd layer. Both Logit Lens and Geometric Lens arrive at \texttt{\_Paris} around the 24th/25th layer, but the tokens elicited by Geometric Lens are more diverse. The full trajectories from Geometric Lens, across all token positions, are visualized via the Laguerre Autoencoder in Figure~\ref{fig:AE_full_Paris}. The first path, generated by the token "The", collapses to a single point, confirming that all its intermediary hidden vectors correspond to the same decision cell. As more tokens are incorporated, the trajectory passes through an increasing number of Laguerre cells. To quantify the diversity of tokens elicited by each lens, we construct a dataset of 249 countries and their capitals (see \S\ref{sec:app_lenses} for details), and map each token to the first principal component of the embedding matrix. Results are shown in Figure~\ref{fig:lenses}. The figure shows that Patchscopes tends to converge to the final token at very early layers, whereas both Geometric Lens and Logit Lens arrive at the final token only in the last several layers; however, Geometric Lens is more stable during the final stage, while exhibiting greater diversity during the early stage. We demonstrate the practical usefulness of these observations in the next section.

\begin{wrapfigure}{r}{0.47\textwidth} % change this value
\vspace{-6mm}
\centering
\subfloat{\includegraphics[width=0.46\textwidth]{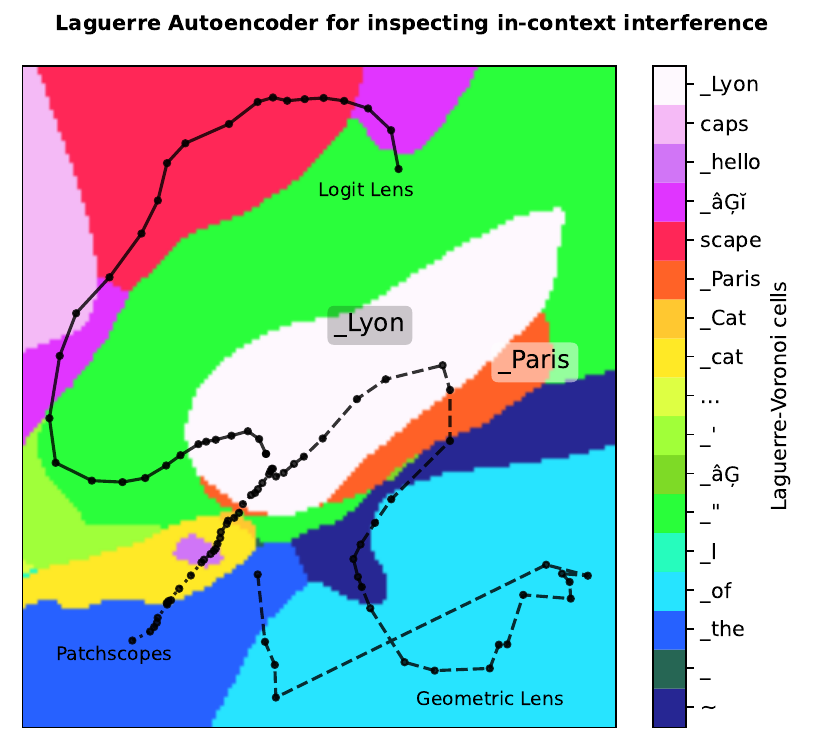}} 
% \subfloat{\includegraphics[height=1.0in]{figs/2D}} 
% \vspace{-3mm}
\caption{Laguerre Autoencoder visualization of the prompt "\texttt{You are in a fictional world where Marseille and Lyon have swapped their names. The Louvre Museum is located in the city of}". Geometric Lens passes through \texttt{\_Paris} before reaching \texttt{\_Lyon}.}
\label{fig:AE2}
\vspace{-2mm} % change this value
\end{wrapfigure}
\textbf{Intermediary Representations Are Not Artifacts, but Carry Useful Meaning.}
Existing lenses operate under the assumption that every intermediary representation should ideally encode the same concept as the final layer. Tuned Lens, for instance, minimizes the KL divergence between the probability distribution produced by the lens at each intermediary layer and the final distribution produced by the frozen model. In this paper, we challenge this assumption, arguing instead that a hidden representation intentionally carries diverse, evolving meanings — and that this diversity can itself be leveraged usefully.

For a token appearing in a lens's output trajectory, we distinguish factual correctness from grammatical correctness. For the prompt "\texttt{The capital of France is}", for example, both "\texttt{\_Paris}" and "\texttt{Paris}" are \textit{factually correct}, but only "\texttt{\_Paris}" (with a leading space) is grammatically correct. We track both tokens across all layers for all three lenses: Logit Lens, Patchscopes, and Geometric Lens. Results are shown in Figure~\ref{fig:paris_AE}(A), where an interesting trend emerges — one visible only through Geometric Lens: in the earlier layers, the ranks of "\texttt{\_Paris}" and "\texttt{Paris}" rise at a similar rate; in the middle layers, "\texttt{Paris}" takes priority, reaching the top-1 rank around layers 20–25; in the final layers, the rank of "\texttt{Paris}" drops sharply, and "\texttt{\_Paris}" emerges as the sole correct token. There is a pronounced peak in "\texttt{Paris}"'s probability between layers 20 and 25, a peak absent for both Logit Lens and Patchscopes. We further use the Laguerre Autoencoder to jointly visualize the trajectories and Laguerre cells in Figure~\ref{fig:paris_AE}(B). This 2D visualization makes clear that the three lenses follow markedly different trajectories: Patchscopes takes the shortest path of the three; Logit Lens passes through fewer cells; and Geometric Lens traces the most complex path, passing through the greatest number of regions. Notably, Geometric Lens is the only lens whose trajectory starts from generic tokens ("\_a" and "\_the"), passes through the factually correct token "\texttt{Paris}", and finally settles in the "\texttt{\_Paris}" cell — consistent with the intuition that an LLM gradually refines its answer across layers.

This suggests that the LLM first retrieves all factually correct candidates in its earlier layers, and only later resolves both factual and grammatical correctness. Geometric Lens is the only lens that reveals this mechanism.
%Please add the following packages if necessary:
%\usepackage{booktabs, multirow} % for borders and merged ranges
%\usepackage{soul}% for underlines
%\usepackage{xcolor,colortbl} % for cell colors
%\usepackage{changepage,threeparttable} % for wide tables
\begin{table}[!htp]\centering\small
\caption{Comparison of different lenses on the task of recovering the correct token under in-context interference.}\label{tab:lens}
\resizebox{\textwidth}{!}{ % use this if the table is too large
\begin{tabular}{lrrr|rrr|rrr|rrr}\toprule
Models &\multicolumn{3}{c}{Phi-2 } &\multicolumn{3}{c}{Gemma-2-2B} &\multicolumn{3}{c}{Gemma-2-9B} &\multicolumn{3}{c}{Llama-3.1-8B} \\\cmidrule{1-13}
Datasets &Paris &London &Berlin &Paris &London &Berlin &Paris &London &Berlin &Paris &London &Berlin \\\midrule
Greedy (top-1) &0.00 &0.01 &0.00 &0.03 &0.15 &0.19 &0.00 &0.03 &0.02 &0.01 &0.02 &0.00 \\
Logit Lens &0.05 &0.15 &0.00 &0.07 &0.52 &0.16 &0.32 &0.01 &0.24 &0.07 &0.19 &0.01 \\
Patchscopes &0.55 &0.25 &0.13 &0.09 &0.39 &0.27 &\textbf{0.41} &0.03 &0.29 &0.02 &0.03 &0.02 \\
Geometric Lens &\textbf{0.56} &\textbf{0.31} &\textbf{0.49} &\textbf{0.25} &\textbf{0.61} &\textbf{0.59} &\textbf{0.41} &\textbf{0.06} &\textbf{0.50} &\textbf{0.13} &\textbf{0.28} &\textbf{0.36} \\
\bottomrule
\end{tabular}}
\end{table}

\textbf{Geometric Lens Reveals a Model's Internal Reasoning under In-Context Interference.} \label{sec:in_context}
In this section, we focus on prompts of the form: "\texttt{You are in a fictional world where Marseille and Lyon have swapped their names. The Louvre Museum is located in the city of}", where a prefix mentioning two irrelevant cities A and B precedes a fact about a third city, C. We find that most models answer with the misleading city A or B more than 90\% of the time — a phenomenon we term in-context interference. The question we investigate is which lens, if any, is robust to this interference. To this end, we construct three datasets — \texttt{Paris}, \texttt{London}, and \texttt{Berlin} — each exhibiting in-context interference (see \S\ref{sec:app_lenses} for details). Motivated by the previous section's finding that Geometric Lens's tokens at layers 20–25 tend to be factually correct, we use the token at layer 20, rather than at the final layer, as the model's output, and measure accuracy accordingly. Results are shown in Table~\ref{tab:lens}. Strikingly, all three lenses substantially improve accuracy under this scheme. Geometric Lens achieves accuracies of 0.31–0.56 on Phi-2, 0.25–0.61 on Gemma-2-2B, 0.06–0.50 on Gemma-2-9B, and 0.13–0.36 on Llama-3.1-8B, outperforming both Patchscopes and Logit Lens in each case. We further investigate the token-switching phenomenon; an example from the \texttt{Paris} dataset is shown in Figures~\ref{fig:AE2_top_5} and~\ref{fig:AE2}. Among the three lenses, Geometric Lens is the only one for which the correct token \texttt{\_Paris} ever attains the top-1 rank at a middle layer. The Laguerre Autoencoder visualization further shows that Geometric Lens is the only lens whose trajectory passes through the \texttt{\_Paris} cell before finally settling in the \texttt{\_Lyon} cell. These findings suggest that, although the final output token is always the same, different lenses can elicit markedly different trajectories at middle layers — and only the correct lens reveals the model's true, and often complex, internal reasoning process.

\section{Discussion and Conclusion} \label{sec:discussion}
In this work, we unify several distinct directions in LLM research — the Linear Representation Hypothesis (LRH), concept geometry, piecewise-linear subdivision, lens-based hidden-representation elucidation, and in-context learning — within a single, holistic Laguerre Geometry framework. In doing so, we challenge two prevailing assumptions: the vector-based definition of a concept under LRH, and the absence of a principled ground truth for hidden representations.

Beyond LRH, several more elaborate hypotheses have been proposed, including the Frame Representation Hypothesis~\citep{valois2025frame}, the Lattice Representation Hypothesis~\citep{xiong2026lattice}, and the Platonic Representation Hypothesis~\citep{bangachev2026representation}. Rather than proposing yet another such hypothesis, we instead establish a geometric framework in which LRH can be strictly defined and shown to arise naturally.

A strict definition of concept is a prerequisite for grounding LRH in LLMs. By extending the piecewise-linear subdivision theory of DNNs to LLMs, we uncover their piecewise-linear structure and redefine concepts as regions within it. Regions — both the "outer" Laguerre-Voronoi cells and the "inner" linear regions — constitute the most basic geometric elements of each layer in an LLM (Theorem~\ref{thm:dnn_piecewise}), and carry substantially richer geometric information than directions or vectors alone.

Lenses have long remained heuristic, lacking a clear theory of what the "ground truth" for a hidden representation should be. In this paper, we argue that this ground-truth label should instead be derived from the piecewise-linear subdivision of the model itself. Our theory assigns a label to each response region, thereby making it possible to label every hidden vector contained within it. This labeling motivates a new lens, the Geometric Lens, which we validate empirically both in regular forward passes and under in-context interference.

This paper lays a theoretical foundation for further examining how concepts are organized, using response regions as basic building blocks. As an initial step in establishing this geometric framework, we have not yet addressed causal intervention — that is, the controlling and steering of model behavior~\citep{gurnee2026verbalizable}. In future work, we plan to apply this framework to behavior steering in LLMs.

% end of text

\bibliographystyle{unsrtnat}
\bibliography{references}  %%% Uncomment this line and comment out the ``thebibliography'' section below to use the external .bib file (using bibtex) .

\newpage
\appendix
\section{Appendix}
\subsection{Proofs} 
In this section, we present the proof of the theorems.
\subsubsection{Proof of Theorem 4.1} 
\label{sec:app_pd}

\begin{lemma}
\label{lem:power}
The vertical projection of the lower envelope of the hyperplanes
$\{\Pi_k(\vz): \mW_k^\top \vz + \nu_k\}_{k=1}^K$
onto the input space $\mathbb{R}^n$ induces the cells of an LVD.
\end{lemma}

\noindent\textbf{Theorem 4.1} (Unembedding Layer is a Laguerre-Voronoi Diagram). 
    The unembedding layer of an LLM, parameterized by $\mW, \vb$, partitions the representation space $\R^d$ into a Laguerre-Voronoi Diagram with centers $\{\vc_1,...,\vc_j,...,\vc_V\}$ and weights $\{ \nu_1,...,\nu_j,...,\nu_V \}$ given by:
	\begin{equation}
		\vc_j = \tfrac{1}{2} \mW_j,
	\end{equation}
	\begin{equation}
		\nu_j = \vb_j + \tfrac{1}{4} \|\mW_j\|_2^2,
	\end{equation}
    where $V = |\gV|$, i.e., the number of cells equals the vocabulary size of the LLM.
\begin{proof}
We first establish Lemma~\ref{lem:power} by deriving the correspondence between a hyperplane $\Pi_k(\vz)$ and the center of its associated cell in $\mathbb{R}^n$. By definition, the assignment of a point $\vz\in\mathbb{R}^n$ to a cell is determined by comparing the quantities $d(\vz,\vc_k)^2-\nu_k$ across all centers. Accordingly, we define the power function
\begin{equation}
\label{eq:powerfun}
    p(\vz,S)=\|\vz-\vc\|_2^2-r^2,
\end{equation}
where $S\subseteq\mathbb{R}^n$ is a sphere with center $\vc$ and radius $r$. Equivalently, the weight associated with a center can be interpreted as the squared radius, i.e., $\nu=r^2$.

Next, let $U$ denote the paraboloid $y=\|\vz\|_2^2$, and let $\Pi(S)$ denote the transformation that maps a sphere $S$ with center $\vc$ and radius $r$ to the hyperplane
\begin{equation}
\label{eq:pifun}
    \Pi(S):\quad y=2\vz^\top\vc-\vc^\top\vc+r^2.
\end{equation}
Aurenhammer~\citep{aurenhammer1987power} showed that $\Pi$ is a bijection between arbitrary spheres in $\mathbb{R}^n$ and nonvertical hyperplanes in $\mathbb{R}^{n+1}$ that intersect $U$.

Let $\vz'$ denote the vertical projection of $\vz$ onto $U$, and let $\vz''$ denote its vertical projection onto $\Pi(S)$. The power function can then be expressed as
\begin{equation}
\label{eq:pifun2}
    p(\vz,S)=d(\vz,\vz')-d(\vz,\vz'').
\end{equation}
This characterization implies the following result (Lemma 4 in~\citep{aurenhammer1987power}): for two non-concentric spheres $S_1$ and $S_2$ in $\mathbb{R}^n$, the bisector of their power cells is given by the vertical projection of $\Pi(S_1)\cap\Pi(S_2)$ onto $\mathbb{R}^n$.

We therefore obtain a one-to-one correspondence between a sphere $S$ and its associated hyperplane $\Pi(S)$. Comparing Eq.~(\ref{eq:pifun}) with the hyperplanes induced by the unembedding layer,
$\{\Pi_k(\vz): \mW_k^\top\vz+\vb_k\}_{k=1}^K$,
yields
\begin{equation}
\label{eq:mapping}
\begin{aligned}
    \vc &= \frac{1}{2}\mW_k,\\
    \nu=r^2 &= \vb_k+\frac{1}{4}\|\mW_k\|_2^2.
\end{aligned}
\end{equation}

Although $\nu$ admits the geometric interpretation of a squared radius, i.e., $\nu=r^2$, it is not necessarily nonnegative for an arbitrary unembedding layer (or, more generally, an arbitrary multiclass linear classifier), since $\vb_k$ may take sufficiently negative values.
\end{proof}

\subsubsection{Proof of Theorem 4.3}

\noindent\textbf{Theorem 4.3} (Linear Separability of Categories). \label{sec:app_category}
    Any two categories $U_{\sM} = \bigcup_{m=1}^M \omega_m$ and $U_{\sN} = \bigcup_{n=1}^N \omega_n$ are linearly separable if and only if there exist a vector $\vv \in \R^d$, a scalar $\alpha \in \R$, a margin $\epsilon \in \R_{>0}$, and sets of non-negative multipliers $\lambda_{m}, \mu_{n} \in \R_{\ge0}^{V-1}$, for all $m \in \{1,...,M\}$ and $n \in \{1,...,N\}$, such that:\\
$A^{(m)T}\lambda_m=\vv, \; \beta^{(m)T}\lambda_m\le \alpha-\epsilon$, where
$A_j^{(m)}=2(\vc_j-\vc_m)^T, \; \beta_j^{(m)}=\|\vc_j\|^2-\nu_j-\|\vc_m\|^2+\nu_m, \; \forall j \neq m$, and\\
$A^{(n)T}\mu_n=-\vv, \; \beta^{(n)T}\mu_n\le -\alpha-\epsilon$, where
$A_j^{(n)}=2(\vc_j-\vc_n)^T, \; \beta_j^{(n)}=\|\vc_j\|^2-\nu_j-\|\vc_n\|^2+\nu_n, \; \forall j \neq n$.

\begin{proof}
A cell associated with the site--weight pair $(\vc_m,\nu_m)$ is defined as
\[
C_m=\{\vz\in\R^d:\|\vz-\vc_m\|^2-\nu_m<\|\vz-\vc_j\|^2-\nu_j,\ \forall j\neq m\},
\]
which can be equivalently written as
\[
C_m=\{\vz\in\R^d:2(\vc_j-\vc_m)^T\vz<\|\vc_j\|^2-\nu_j-\|\vc_m\|^2+\nu_m,\ \forall j\neq m\}.
\]

Define
\[
A_j^{(m)}=2(\vc_j-\vc_m)^T,\qquad
\beta_j^{(m)}=\|\vc_j\|^2-\nu_j-\|\vc_m\|^2+\nu_m,\qquad
\forall j\neq m.
\]
Then,
\[
C_m=\{\vz:A^{(m)}\vz\le\beta^{(m)}\}.
\]

Suppose that $\vv$ defines a separating hyperplane between $U_{\sM}$ and $U_{\sN}$, i.e.,
\[
\vv^T\vz\le\alpha-\epsilon,\quad\forall\vz\in U_{\sM},
\]
and
\[
\vv^T\vz\ge\alpha+\epsilon,\quad\forall\vz\in U_{\sN}.
\]

Since $U_{\sM}=\bigcup_{m=1}^M\omega_m$, the first condition is equivalent to
\[
\vv^T\vz\le\alpha-\epsilon,\quad
\forall\vz\in\omega_m,\;
\forall m\in\{1,\ldots,M\}.
\]
Applying Farkas' lemma to the implication
\[
A^{(m)}\vz\le\beta^{(m)}
\;\Longrightarrow\;
\vv^T\vz\le\alpha-\epsilon,
\]
shows that it holds if and only if there exists $\lambda_m\ge0$ such that
\[
A^{(m)T}\lambda_m=\vv,\qquad
\beta^{(m)T}\lambda_m\le\alpha-\epsilon.
\]

Similarly, for every $n\in\{1,\ldots,N\}$,
\[
A^{(n)}\vz\le\beta^{(n)}
\;\Longrightarrow\;
\vv^T\vz\ge\alpha+\epsilon,
\]
holds if and only if there exists $\mu_n\ge0$ such that
\[
A^{(n)T}\mu_n=-\vv,\qquad
\beta^{(n)T}\mu_n\le-\alpha-\epsilon.
\]
\end{proof}

\subsubsection{Proof of Theorem 4.4}

\noindent\textbf{Theorem 4.4} (Concept Domination). \label{sec:app_concept}
    For two concepts $C_j(\omega_j, \vc_j, \nu_j)$ and $C_{j'}(\omega_{j'}, \vc_{j'}, \nu_{j'})$, $C_j$ dominates $C_{j'}$ (i.e., $\omega_{j'} \subseteq \omega_j$) if $\text{Domination}(j \to j') > 0$.

\begin{proof}
For any point $\vz$, concept $C_j$ dominates $\vz$ over concept $C_{j'}$ if
\[
\|\vz-\vc_j\|^2-\nu_j
<
\|\vz-\vc_{j'}\|^2-\nu_{j'}.
\]

Setting $\vz=\vc_{j'}$ gives
\[
\|\vc_{j'}-\vc_j\|^2-\nu_j
<
-\nu_{j'},
\]
or equivalently,
\[
\|\vc_{j'}-\vc_j\|^2
<
\nu_j-\nu_{j'}.
\]

Since
\[
-\nu_{j'}
=
\min_{\vz}
\left(\|\vz-\vc_{j'}\|^2-\nu_{j'}\right)
\]
is the lowest possible Laguerre distance to $\vc_{j'}$, the above inequality implies that even at the point minimizing the Laguerre distance to concept $C_{j'}$, the Laguerre distance to $C_j$ is strictly smaller. Therefore, every point assigned to $C_{j'}$ is also assigned to $C_j$, implying $\omega_{j'}\subseteq\omega_j$.
\end{proof}

\subsubsection{Proof of Theorem 4.6} \label{sec:app_piecewise}

We first establish that shortcut connections preserve piecewise-affine structure.

\begin{lemma}[Piecewise Affinity of Shortcut Connections]
\label{thm:shortcut}
If $g(\vx)$ is piecewise affine, then $f(\vx)=g(\vx)+\vx$ is also piecewise affine.
\end{lemma}

\begin{proof}
Suppose $g$ is piecewise affine. Then, for each activation region $\Omega_i$, there exist a constant matrix $\mA_i$ and vector $\vc_i$ such that
\[
g(\vx)=\mA_i\vx+\vc_i,\qquad \vx\in\Omega_i.
\]
Within the same region,
\[
f(\vx)
=
g(\vx)+\vx
=
(\mA_i+\mI)\vx+\vc_i,
\]
which is affine. Since this holds for every activation region, $f$ is piecewise affine.

More generally, let
\[
F=F_2\circ f\circ F_1,
\]
where $F_1$ and $F_2$ are piecewise affine. Because the composition of piecewise-affine mappings is itself piecewise affine, $F$ remains piecewise affine.
\end{proof}

We now prove that a single-token transformer induces a piecewise-affine partition at every layer.

\noindent\textbf{Theorem 4.6} (Single-input Transformer Piecewise-linear Partition).
    Let $F$ be an $L$-layer transformer model with feedforward networks $\{\text{MLP}^{(\ell)}\}$, each using a piecewise-linear activation function; value matrices $\{\mW^{V(\ell)}\}$ and output matrices $\{\mW^{O(\ell)}\}$, for $\ell \in \{1,...,L\}$; and a language modeling head $F_{un}$, given by:
\begin{align}
        F(\vx) &= F_{un} \circ M_L \circ A_L \circ \cdots \circ M_{\ell} \circ A_{\ell} \circ \cdots \circ M_1 \circ A_1(\vx), \\
        M_{\ell}(\vx) &= \vx + \text{MLP}^{(\ell)}(\vx), \\
        A_{\ell}(\vx) &= \vx + \vx \mW^{V(\ell)} \mW^{O(\ell)}.
\end{align}
    Then $F$ partitions the input space, as well as each intermediary representation space, into piecewise-linear regions.
\begin{proof}
Consider a single input token. In this case, each attention head contains only one key-value pair, so the attention weight is identically equal to one. Consequently, the output of the multi-head attention module is simply the concatenation of the value vectors followed by the output projection:
\[
\Delta\mathrm{attn}^{(\ell)}
=
\mathrm{Concat}
\!\left(
\vv^{(\ell,1)},
\dots,
\vv^{(\ell,H)}
\right)
\mW^{O(\ell)},
\]
where
\[
\vv^{(\ell)}
=
\vx^{(\ell)}
\mW^{V(\ell)}.
\]
Therefore,
\[
\Delta\mathrm{attn}^{(\ell)}
=
\vx^{(\ell)}
\mW^{V(\ell)}
\mW^{O(\ell)},
\]
and the attention block
\[
A_\ell(\vx)
=
\vx+\vx\mW^{V(\ell)}\mW^{O(\ell)}
\]
is an affine transformation.

Each feedforward block employs a piecewise-linear activation function and is therefore piecewise affine. By Lemma~\ref{thm:shortcut}, adding the residual connection preserves piecewise affinity, implying that $M_\ell$ is piecewise affine.

Finally, since every layer is the composition of an affine attention block and a piecewise-affine feedforward block, and compositions of piecewise-affine mappings remain piecewise affine, every intermediate representation space is partitioned into piecewise-affine regions. The language modeling head is affine, so the complete transformer $F$ also induces a piecewise-affine partition of the input space.
\end{proof}

\noindent\textit{Remark on RMSNorm and RoPE.}
Although other normalization methods (e.g., LayerNorm~\citep{ba2016layer} and BatchNorm~\citep{ioffe2015batch}) and positional encodings (e.g., absolute positional encoding) exist, we focus on RMSNorm and RoPE.

Ignoring the numerical stabilization constant $\epsilon$, parameter-free RMSNorm is given by
\[
\mathrm{RMSNorm}(\vx)
=
\frac{\vx}{\sqrt{\frac{1}{d}\|\vx\|_2^2}},
\]
which rescales each hidden representation along its radial direction. Although RMSNorm is not affine, it preserves the assignment of representations to the piecewise-affine regions established by the preceding computations, affecting only their geometric embedding.

RoPE modifies only the query and key projections, $\mW^Q$ and $\mW^K$, while leaving $\mW^V$ and $\mW^O$ unchanged. Since the single-token attention block depends only on $\mW^V$ and $\mW^O$, RoPE does not alter the induced partition.

% change: \mW_{\ell}^V to \mW^{V{(\ell)}}
% change: \mW_{\ell}^O to \mW^{O{(\ell)}}
\begin{figure}
    \centering
    \subfloat{\includegraphics[height=5.3cm]{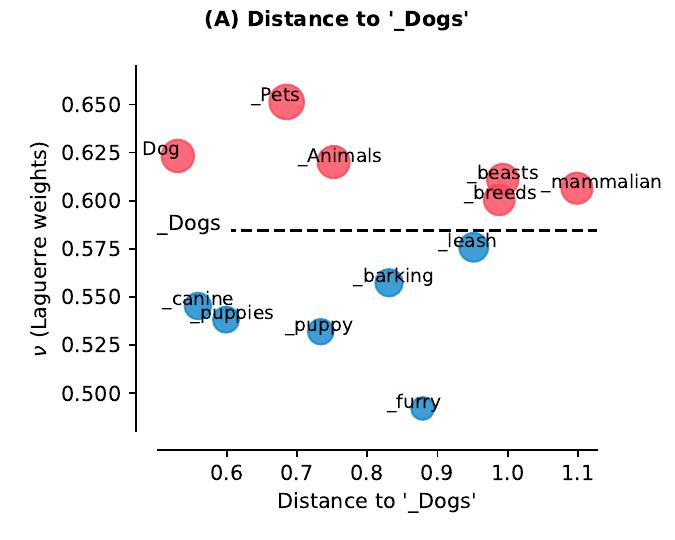}} 
    \subfloat{\includegraphics[height=5.3cm]{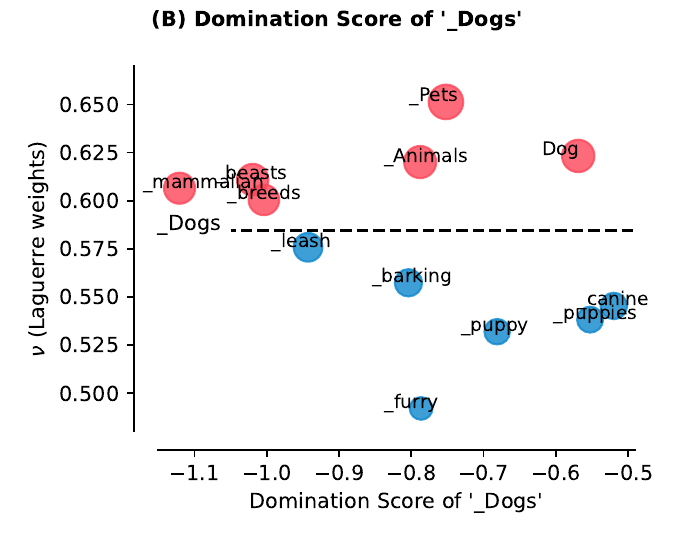}} \\
    \subfloat{\includegraphics[height=5.3cm]{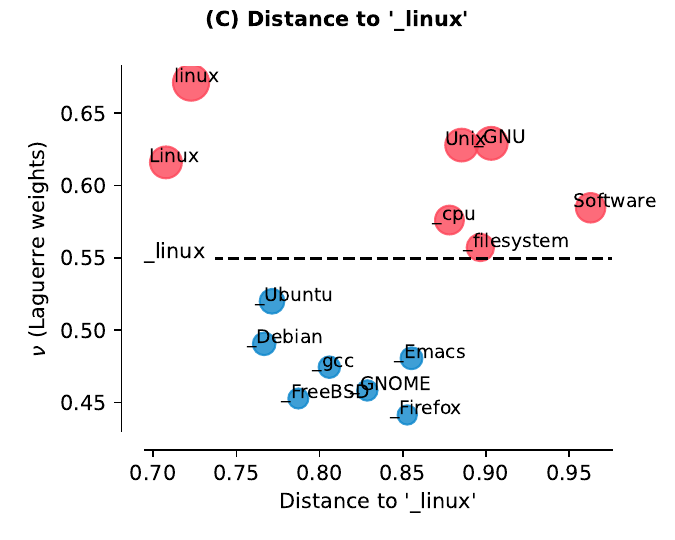}} 
    \subfloat{\includegraphics[height=5.3cm]{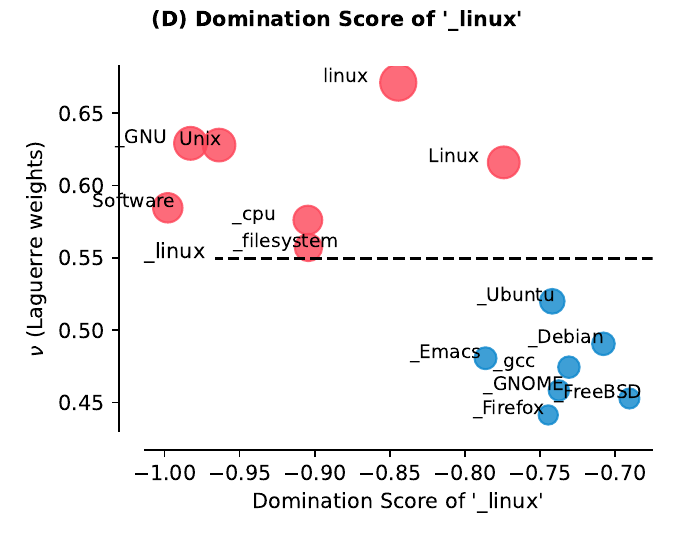}} 
    \caption{(A) Use distance to the embedding vector of "\_Dogs" to retrieve its hyponyms/hypernyms. (B) Use the domination score of "\_Dogs" to retrieve its hyponyms/hypernyms. (C) Use distance to the embedding vector of "\_linux" to retrieve its hyponyms/hypernyms. (D) Use the domination score of "\_linux" to retrieve its hyponyms/hypernyms. The embedding vectors, together with the Laguerre weights, jointly distinguish hyponyms/hypernyms.}\label{fig_supp:inclusion}
\end{figure}
\begin{figure}
    \centering
    \subfloat{\includegraphics[width=0.8\paperwidth]{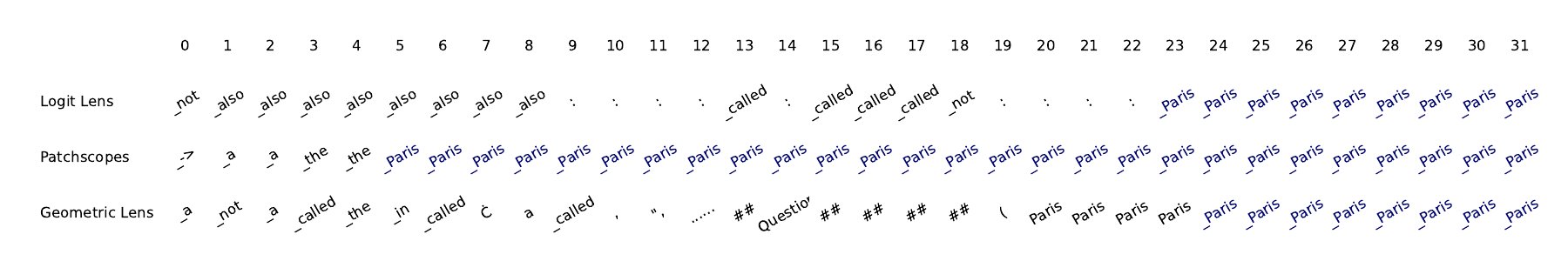}} 
    \caption{The token revealed by different lenses at each layer for the prompt: "The capital of France is". Three methods generate different token trajectories along the layers.}\label{fig:paris}
    % \vspace{-5mm}
\end{figure}
\begin{figure}
    \centering
    \subfloat{\includegraphics[width=0.5\paperwidth]{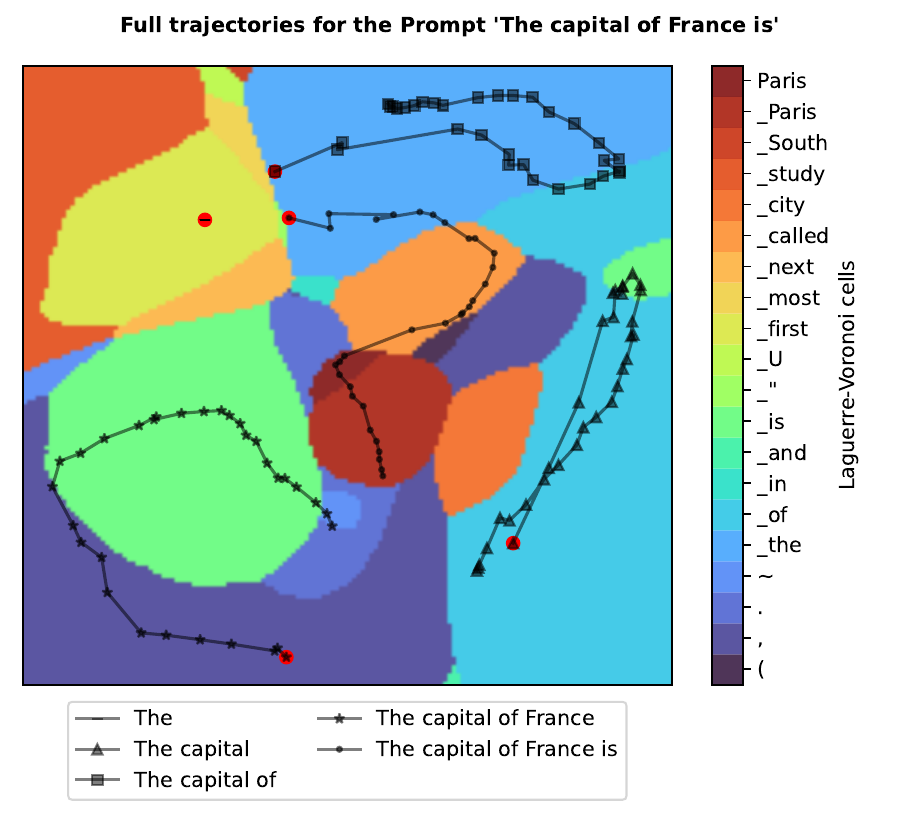}} 
    \caption{The full trajectories from the Geometric Lens for all token positions for the prompt: "\texttt{The capital of France is}". The first trajectory collapses into a single point.}\label{fig:AE_full_Paris}
    % \vspace{-5mm}
\end{figure}
\begin{figure}
    \centering
    \subfloat{\includegraphics[width=\textwidth]{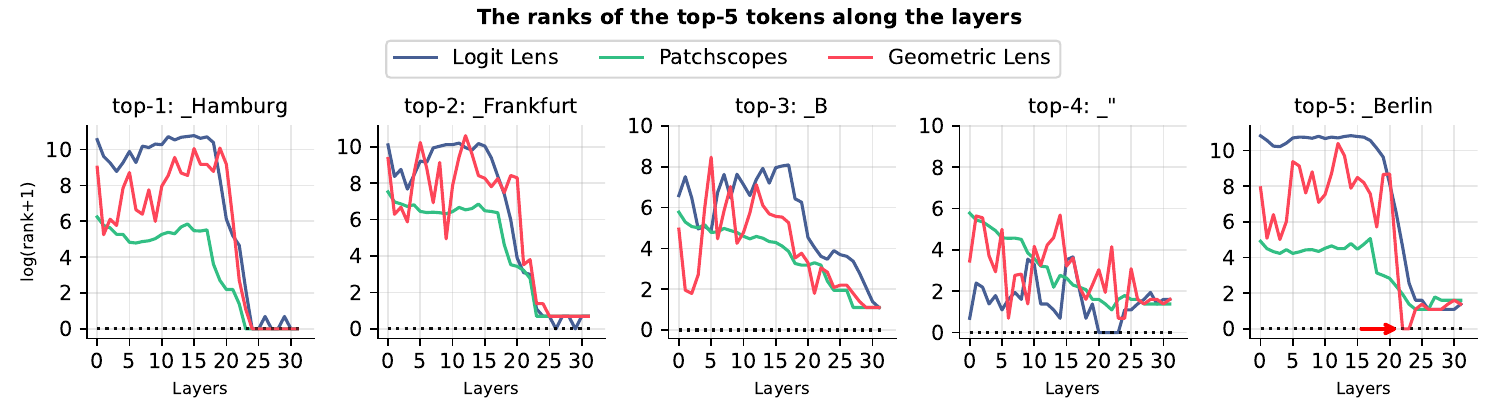}}
    \caption{The monitored log ranks of the top-5 final tokens for the prompt "\texttt{You are in a fictional world where Hamburg and Frankfurt have swapped their names. The Bode Museum is located in the city of}".}\label{fig:AE_berlin_top_5}
\end{figure}
\begin{figure}
    \centering
    \subfloat{\includegraphics[width=\textwidth]{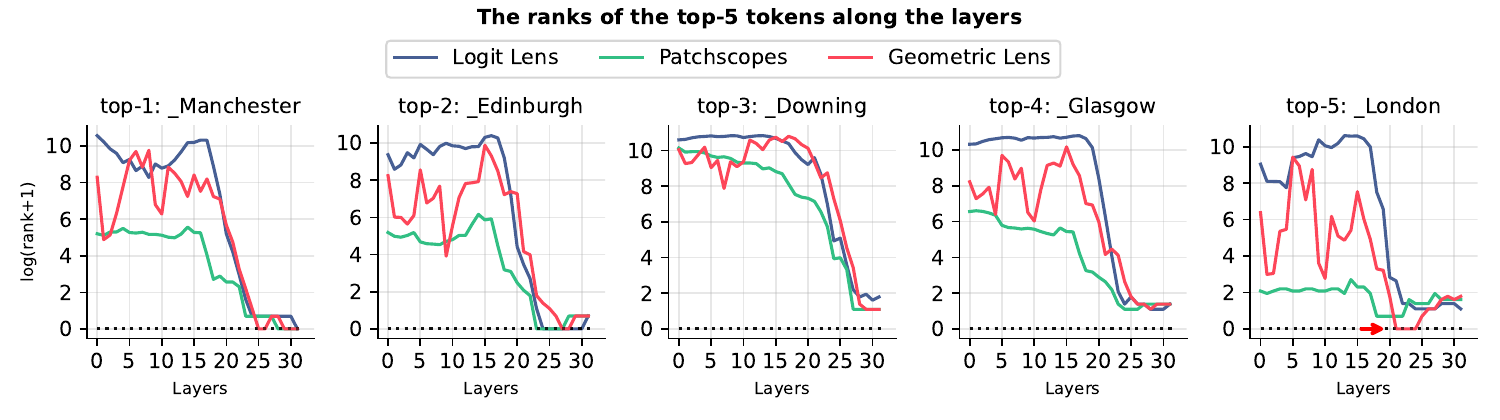}}
    \caption{The monitored log ranks of the top-5 final tokens for the prompt "\texttt{You are in a fictional world where Edinburgh and Manchester have swapped their names. Downing Street is located in the city of}".}\label{fig:AE_london_top_5}
\end{figure}
\begin{figure}
    \centering
    \subfloat{\includegraphics[height=8.3cm]{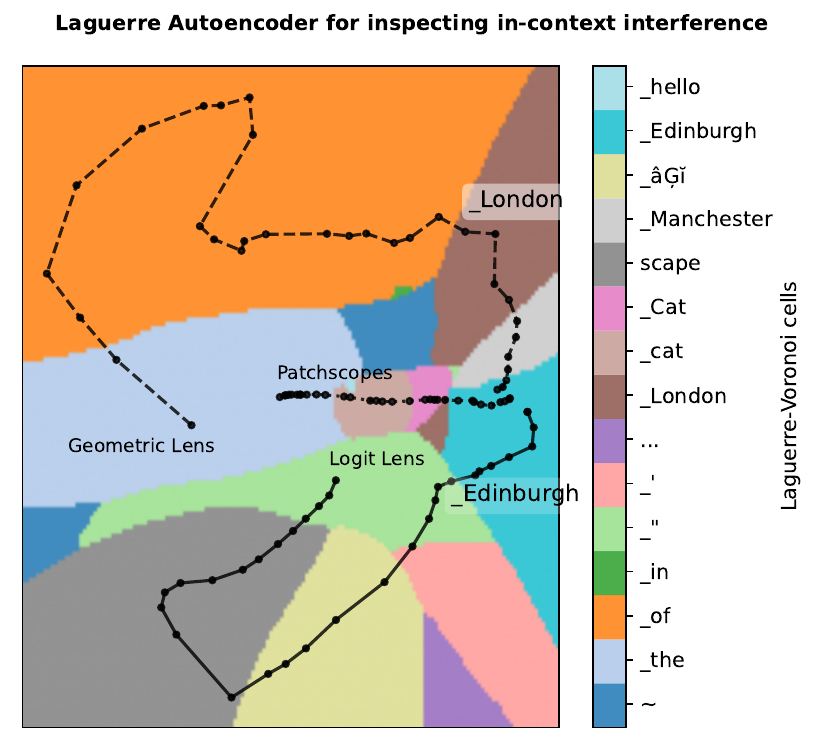}} 
    \caption{The Laguerre Autoencoder visualization of the prompt: "\texttt{You are in a fictional world where Edinburgh and Manchester have swapped their names. Downing Street is located in the city of}".}\label{fig:AE_london}
    \vspace{-5mm}
\end{figure}
\begin{figure}
    \centering
    \subfloat{\includegraphics[height=8.3cm]{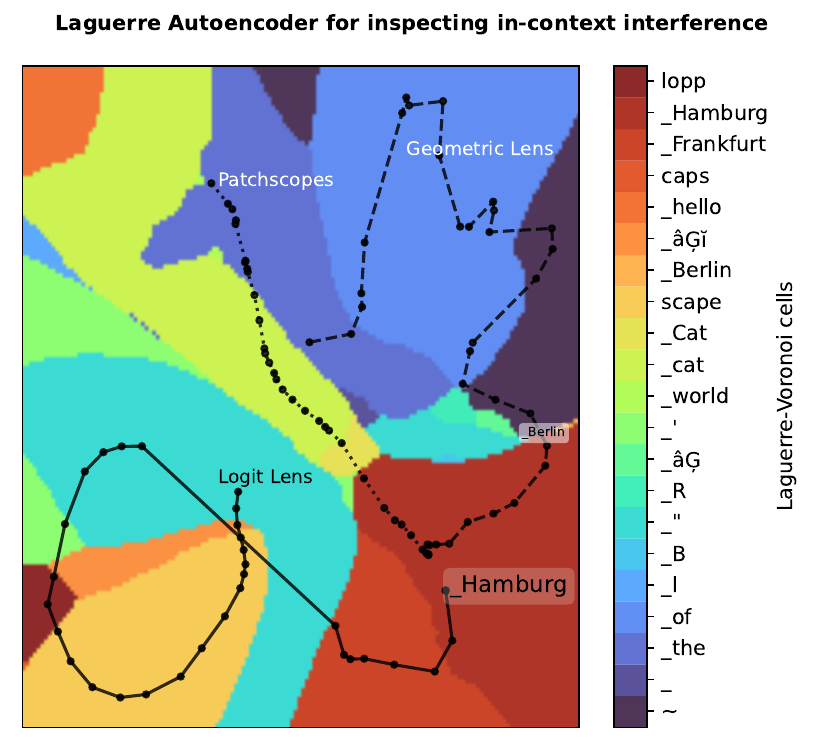}} 
    \caption{The Laguerre Autoencoder visualization of the prompt: "\texttt{You are in a fictional world where Hamburg and Frankfurt have swapped their names. The Bode Museum is located in the city of}".}\label{fig:AE_berlin}
    \vspace{-5mm}
\end{figure}

\subsection{Experiment Details} \label{sec:app_exp}

\textbf{Generation of Piecewise-Linear Subdivisions in the 2D Space.} \label{sec:app_2d}
We generated three synthetic distributions in a two-dimensional input space: one Gaussian distribution and two circular distributions. A one-hidden-layer ReLU network with 12 hidden neurons was trained to classify the three classes. The network parameters were initialized from a Gaussian distribution with mean 0 and standard deviation 0.5. Training was performed using SGD with a learning rate of 0.1 and a batch size of 32.

Each hidden neuron partitions the input space into two half-spaces (gray lines in the figure), while the final decision boundary between two classes is projected onto the input space as a piecewise-linear curve composed of linear segments (red segments). Together, the hidden-layer partition and the classification polytope partition the input space into piecewise-linear regions.

Unlike previous work, which quantifies the number of linear regions as a measure of network expressivity, we instead study the semantic label associated with each region. Specifically, we assign each linear region $R$ a label $C(R)$ corresponding to the label of the final classification Laguerre--Voronoi cell projected onto that region (e.g., the green region in the figure). This labeling induces a hierarchical tree structure consisting of $V$ trees, each rooted at a top-level Laguerre--Voronoi cell, where $V$ denotes the number of classes. The resulting structure enables us to trace both intra-tree transitions and inter-region (or equivalently, inter-tree) transitions. A similar tree representation was previously illustrated in Figure 2(c) of \citet{montufar2014number}.

\textbf{Generation of Category Data from the WordNet Hierarchy.} \label{sec:app_category_data}
We did not directly use the datasets constructed by \citet{xiong2026lattice}, as they are not publicly available. Instead, we followed their data construction procedure with several refinements.
Given a root synset (e.g., \texttt{animal}), we recursively extract its subtree by traversing WordNet hyponym relations. This process produces a candidate set of words whose hypernyms descend from the selected root concept. We then perform a second filtering step to ensure that the primary sense of each word belongs to the target subtree. For example, the word \texttt{schoolmaster} appears as a hyponym of \texttt{animal}; however, its three senses are ``presiding officer of a school,'' ``educator,'' and ``food fish,'' only the last of which belongs to the animal hierarchy. Such words are excluded to obtain clean, semantically coherent, and non-overlapping concept sets.

Using this procedure, we construct four concept datasets corresponding to \texttt{animal}, \texttt{plant}, \texttt{cognition}, and \texttt{event}. For each language model, we further filter these datasets by retaining only words that correspond to a single vocabulary token under the model's tokenizer. For example, for Phi-2, the resulting datasets contain 93 animal concepts, 48 plant concepts, 498 cognition concepts, and 1,302 event concepts. Because different language models employ different tokenizers, these counts may vary across models.
Following prior work~\citep{park2025geometry}, we consider only vocabulary tokens with a leading whitespace, denoted by ``\_'' throughout the paper.

%Please add the following packages if necessary:
%\usepackage{booktabs, multirow} % for borders and merged ranges
%\usepackage{soul}% for underlines
%\usepackage{xcolor,colortbl} % for cell colors
%\usepackage{changepage,threeparttable} % for wide tables
\begin{table}[!htp]\centering\small
\caption{Error rates for linear probing experiments with different definitions of concepts by Categorical Geometry~\citep{park2025geometry}, Lattice Geometry~\citep{xiong2026lattice}, and Laguerre Geometry (ours).}\label{tab:linearity2}
\resizebox{\textwidth}{!}{ % use this if the table is too large
\begin{tabular}{lrrrrrrrrr}\toprule
&\multicolumn{4}{c}{Pythia-70M} &\multicolumn{4}{c}{Gemma-3-270M} \\\cmidrule{2-9}
dataset &\multicolumn{2}{c}{animal $\leftrightarrow$ plant} &\multicolumn{2}{c}{event $\leftrightarrow$ cognition} &\multicolumn{2}{c}{animal $\leftrightarrow$ plant} &\multicolumn{2}{c}{event $\leftrightarrow$ cognition} \\\cmidrule{1-9}
&error$\downarrow$ &(shuffle) &error$\downarrow$ &(shuffle) &error$\downarrow$ &(shuffle) &error$\downarrow$ &(shuffle) \\\midrule
Categorical Geometry &\textbf{0.167 ± 0.058} &0.498 &0.241 ± 0.022 &0.500 &\textbf{0.101 ± 0.031 } &0.495 &0.237 ± 0.023 &0.499 \\
Lattice Geometry &0.196 ± 0.067 &0.501 &0.251 ± 0.024 &0.502 &0.120 ± 0.035 &0.499 &0.242 ± 0.025 &0.500 \\
Laguerre Geometry &0.179 ± 0.062 &0.502 &\textbf{0.193 ± 0.020} &0.494 &0.128 ± 0.038 &0.499 &\textbf{0.198 ± 0.022} &0.500 \\
\bottomrule
\end{tabular}}
\end{table}
\textbf{Probing Linear Separability under Different Concept Definitions.} \label{sec:app_linear}
For each pair of categories, we randomly sample an equal number of concepts from each category and repeat the experiment 500 times. In every run, the sampled concepts are randomly split into training and test sets using a 60:40 ratio. Concept representations for Categorical Geometry and Lattice Geometry are obtained using their official implementations. For Laguerre Geometry, region representations are computed according to Theorem~\ref{thm:lvd}. A linear classifier is then trained to distinguish 40 concepts sampled from \texttt{animal} and 40 concepts sampled from \texttt{plant}.

Although the concept datasets are constructed to maximize semantic purity, the resulting evaluation remains an approximation of true concept separability because vocabulary tokens may still be polysemous. For example, the token \texttt{\_dog} may represent both the animal \emph{dog} and occurrences within other words such as ``dogma.'' Consequently, token-level representations may not perfectly isolate individual concepts, which is a limitation of works using a token to represent a concept.

%Please add the following packages if necessary:
%\usepackage{booktabs, multirow} % for borders and merged ranges
%\usepackage{soul}% for underlines
%\usepackage{xcolor,colortbl} % for cell colors
%\usepackage{changepage,threeparttable} % for wide tables
\begin{table}[!htp]\centering\small
\caption{Accuracy of hierarchical relations discovered by the Lattice Geometry and the Laguerre Geometry.}\label{tab:domination2}
\resizebox{\textwidth}{!}{ % use this if the table is too large
\begin{tabular}{lrrrrrrrrr}\toprule
Models &\multicolumn{4}{c}{Gemma-2-2B} &\multicolumn{4}{c}{Gemma-3-270M} \\\cmidrule{1-9}
Datasets &Animal &Plant &Event &Cognition &Animal &Plant &Event &Cognition \\\midrule
Lattice Geometry~\citep{xiong2026lattice} &0.4333 &0.4706 &0.5034 &0.5173 &0.5677 &0.5244 &0.5042 &0.4767 \\
(Random) &\footnotesize 0.4833 &\footnotesize 0.5059 &\footnotesize 0.5241 &\footnotesize 0.5159 &\footnotesize 0.5197 &\footnotesize 0.4512 &\footnotesize 0.5181 &\footnotesize 0.5073 \\\midrule
Laguerre Geometry &\textbf{0.5667} &\textbf{0.6118} &\textbf{0.7479} &\textbf{0.6503} &\textbf{0.9607} &\textbf{0.9512} &\textbf{0.9822} &\textbf{0.9535} \\
(Random) &\footnotesize 0.5792 &\footnotesize 0.5774 &\footnotesize 0.5704 &\footnotesize 0.5784 &\footnotesize 0.4802 &\footnotesize 0.4594 &\footnotesize 0.4699 &\footnotesize 0.4589 \\
\bottomrule
\end{tabular}}
\end{table}
\textbf{Experiment Details for Unsupervised Hierarchical Concept Inference in LLMs.} \label{sec:app_domination}
The attribute schema and object--attribute matrix required to compute the soft inclusion score of \citet{xiong2026lattice} are not publicly available. To facilitate a fair comparison with our attribute-free domination score, we approximate the unknown attribute directions using the principal components of the concept vectors when computing the inclusion scores for Lattice Geometry.

Table~\ref{tab:domination2} compares Laguerre Geometry and Lattice Geometry on Gemma-2 and Gemma-3. Our domination score achieves accuracies ranging from 0.5667 to 0.7479 on Gemma-2-2B and from 0.9512 to 0.9822 on Gemma-3-270M. In contrast, Lattice Geometry performs close to random guessing. One possible explanation is that the inclusion score relies on a predefined attribute set, which may encode concept-specific prior information unavailable in our approximation based solely on principal components. 

\textbf{Experiment Details for the Logit Lens, Patchscopes, and Geometric Lens.} \label{sec:app_lenses}
The \texttt{Country} dataset consists of 249 country names paired with their capitals using prompts of the form ``\texttt{The capital of [country] is}.''
For in-context interference, we construct three datasets, \texttt{Paris}, \texttt{London}, and \texttt{Berlin}. For each target city $C$, we first use GPT-5.5 to generate 100 factual prompts related to that city, such as ``\texttt{The Louvre Museum is located in the city of}.'' We then prepend two irrelevant city names, denoted $A$ and $B$, before the fact associated with city $C$. Representative prompts from each dataset are provided below.

\begin{tcolorbox}[colback=gray!10, colframe=gray, title=In-Context Interference (Paris Dataset)]
\texttt{You are in a fictional world where Marseille and Lyon have swapped their names. The Louvre Museum is located in the city of\\
Incorrect tokens: Marseille/Lyon.\\
Correct token: Paris.}
\end{tcolorbox}
\begin{tcolorbox}[colback=gray!10, colframe=gray, title=In-Context Interference (London Dataset)]
\texttt{You are in a fictional world where Edinburgh and Manchester have swapped their names. Downing Street is located in the city of\\
Incorrect tokens: Edinburgh/Manchester.\\
Correct token: London.}
\end{tcolorbox}
\begin{tcolorbox}[colback=gray!10, colframe=gray, title=In-Context Interference (Berlin Dataset)]
\texttt{You are in a fictional world where Hamburg and Frankfurt have swapped their names. The Bode Museum is located in the city of\\
Incorrect tokens: Hamburg/Frankfurt.\\
Correct token: Berlin.}
\end{tcolorbox}
Empirically, prompts following this construction cause the language model to predict either city $A$ or city $B$ instead of the correct city $C$ in over 90\% of cases. We refer to this phenomenon as \emph{in-context interference}. These datasets are used to compare the information revealed by different interpretability lenses under this setting.

We implement Logit Lens and Geometric Lens using PyTorch and the Transformers library. For Patchscopes, we use the official implementation and adopt its default target prompt:
``\texttt{cat->cat; 135->135; hello->hello; ?}.''
Alternative prompts, such as
``\texttt{The multi-tokens present here are }''
and
``\texttt{Hello! Could you please tell me more about },''
are also possible; however, the choice of target prompt is not standardized, and its influence has not been systematically studied. To ensure a fair comparison, we therefore use the default prompt provided by the official implementation.

For the training of the Laguerre Autoencoder, we used a symmetric architecture with encoder and decoder both having 3 layers (512-128-32) using ReLU activation and BatchNorm. We set $\varphi_1 = 1.0$ and $\varphi_2 = 20.0$ to train all LAEs for 50 epochs with a batch size of 32 and a learning rate of 0.001.

Experiments on the \texttt{London} and \texttt{Berlin} datasets exhibit behavior consistent with the results reported in \S\ref{sec:in_context}. As shown in Figures~\ref{fig:AE_london}, \ref{fig:AE_berlin}, \ref{fig:AE_berlin_top_5}, and \ref{fig:AE_london_top_5}, for these cases, Geometric Lens is the only one of the three methods whose trajectory passes through the cell corresponding to the correct token before reaching the final prediction, demonstrating that the observed behavior generalizes across multiple datasets.

\end{document}